\documentclass[lettersize,journal]{IEEEtran}
\usepackage{array}
\usepackage{textcomp}
\usepackage{stfloats}
\usepackage{url}
\usepackage{verbatim}
\usepackage{graphicx}
\usepackage{cite}

\usepackage{amsmath,amsfonts}
\usepackage{amsmath}

\newtheorem{proof}{Proof}[section]
\usepackage[ruled,linesnumbered]{algorithm2e}
\usepackage{hyperref}
\usepackage{enumitem}
\usepackage{multirow}

\usepackage{multicol}
\usepackage{multirow}
\usepackage{subfigure}
\usepackage{bbding}
\usepackage{enumitem}
\usepackage{wasysym}
\usepackage{float}
\usepackage{color}

\begin{document}

\title{Handling Data Heterogeneity in Federated Learning via Knowledge Distillation and Fusion}

\author{Xu Zhou, Xinyu Lei, \IEEEmembership{Member, IEEE}, Cong Yang, Yichun Shi, Xiao Zhang, and Jingwen Shi
\thanks{
This paper is officially published in IEEE Internet of Things Journal, 2023 \cite{zhou2023handling}.

This work was supported in part by the Collaborative Innovation Major Project of Zhengzhou under grant 20XTZX06013 and in part by the National Science Foundation under Grant CNS-2153393 (Corresponding author: Cong Yang).

Xu Zhou and Cong Yang are with the School of Cyber Science and Engineering, Zhengzhou University, Zhengzhou, Henan 450002, China (e-mail: 202022332015352@gs.zzu.edu.cn; wangyuanyc@zzu.edu.cn).

Xinyu Lei is with the Department of Computer Science, Michigan Technological University, Houghton, MI 49931, USA (e-mail: xinyulei@mtu.edu)

Yichun Shi and Jingwen Shi are  with the Department of Computer Science and Engineering, Michigan State University, East Lansing, MI 48824, USA (e-mail: shiyichu@msu.edu; shijingw@msu.edu)

Xiao Zhang is with the Department of Electrical and Computer Engineering, Duke University, Durham, NC, 27708, USA (email: xiao.zhang@duke.edu)
}
}



\maketitle

\begin{abstract}
\label{abstract}
Federated learning (FL) supports distributed training of a global machine learning model across multiple devices with the help of a central server.
However, data heterogeneity across different devices leads to the client model drift issue and results in model performance degradation and poor model fairness. 
To address the issue, we design \textbf{Fed}erated learning with global-local \textbf{K}nowledge \textbf{F}usion (FedKF) scheme in this paper. 
The key idea in FedKF is to let the server return the global knowledge to be fused with the local knowledge in each training round so that the local model can be regularized towards the global optima.
Therefore, the client model drift issue can be mitigated. 
In FedKF, we first propose the active-inactive model aggregation technique that supports a precise global knowledge representation. 
Then, we propose a data-free knowledge distillation (KD) approach to enable each client model to learn the global knowledge (embedded in the global model) while each client model can still learn the local knowledge (embedded in the local dataset) simultaneously, thereby realizing the global-local knowledge fusion process.
The theoretical analysis and intensive experiments demonstrate the superiority of FedKF over previous solutions. 

\end{abstract}


\section{Introduction}
\label{introduction}

\IEEEPARstart{F}{ederated} learning (FL) supports distributed training of a global machine learning model across multiple devices with the help of a central server.
The local dataset held by each device is never exchanged in FL, so the local dataset privacy is protected.
The most famous FL algorithm is FedAvg \cite{mcmahan2017communication}.
In one round of FedAvg training, a central server sends the global model weight to a portion of distributed devices (i.e., active clients).
Then, each device trains the model using the local data.
Next, each device sends the new model weight to the central server, which is responsible for computing the new aggregated averaged model.
Afterward, the server sends the new global model to some re-selected active clients to start the next round of FedAvg training.
After numerous rounds of training, the ML model can be well-trained.
Each device can exploit the well-trained ML model for different applications.
In this paper, we refer to devices as clients in FL.

In recent years, FL has been intensively studied in both
academia and industry \cite{zheng2022exploring, wu2022fedadapt, vu2022joint, al2022towards,liu2022distributed}.
One problem rooted in FL is that the heterogeneity of data across different clients can lead to significant model performance (i.e., test accuracy) degradation \cite{karimireddy2020scaffold, wang2020tackling, wang2021field}.
For example, multiple hospitals would like to collaboratively train a disease diagnosis ML model via FL.
Each hospital collects patient data independently, so their datasets are unbalanced and non-IID (i.e.,  heterogeneous).
One hospital might have few/no data samples belonging to a certain disease class.
When each hospital trains the model locally, its local objective may be far from the global objective.
Thus, the averaged global model can be away from the global optima.
This phenomenon is called client model drift in some literatures \cite{li2020federated, lin2020ensemble, zhu2021data, yao2021local}.
The client model drift may lead to poor global model performance.

Except for poor model performance, data heterogeneity may also lead to poor model fairness.
Consider a disease diagnosis ML model that is trained by multiple different hospitals via FL.
These hospitals collect patient data in different geographical areas with different race populations.
The FL-trained ML model can achieve high averaged model performance, but it may have large model performance variance across different hospitals (when testing on their local patient datasets).
If so, the ML model has biases against certain geographical areas (i.e., geographical discrimination) and race populations (i.e., race discrimination).
To mitigate such a model bias issue, the FL-trained ML model should also achieve high fairness, which can be measured by the degree of uniformity in model performance across different clients.
The considered concept of fairness is also named accuracy parity \cite{li2019fair}.

\begin{table*}[ht]
  \centering
  \caption{Comparison among different previous solutions and  ($\CIRCLE$: "Yes", $\Circle$: "No").}
  \label{tab:CMP}
  \resizebox{2\columnwidth}{!}{
  \begin{tabular}{cc|c|c|c|c}
  \hline
    \multicolumn{2}{c|}{\multirow{2}{*}{Solutions}} & Over FedAvg & \multirow{2}{*}{Fairness-Aware} & No Proxy & No Additional \\
    & & Model Performance & & Data Required & Info. Leakage \\
    \hline
    \hline
    \multicolumn{1}{c|}{\multirow{7}{*}{Model performance-based}} & FedAvg \cite{mcmahan2017communication} & $\Circle$  & $\Circle$  & $\CIRCLE$ & $\CIRCLE$ \\
    \cline{2-6}
    \multicolumn{1}{c|}{} & FedGen \cite{zhu2021data} & $\CIRCLE$ & $\Circle$ & $\CIRCLE$ & $\Circle$  \\
    \cline{2-6}
    \multicolumn{1}{c|}{} & CCVR \cite{luo2021no} & $\CIRCLE$ & $\Circle$ & $\CIRCLE$ & $\Circle$ \\
    \cline{2-6}
    \multicolumn{1}{c|}{} & FedDF \cite{lin2020ensemble} & $\CIRCLE$ & $\Circle$ & $\Circle$ & $\CIRCLE$ \\
    \cline{2-6}
    \multicolumn{1}{c|}{} & FedProx \cite{li2020federated} & $\CIRCLE$ & $\Circle$ & $\CIRCLE$ & $\CIRCLE$ \\
    \cline{2-6}
    \multicolumn{1}{c|}{} & MOON \cite{li2021model} & $\CIRCLE$ & $\Circle$ & $\CIRCLE$ & $\CIRCLE$ \\
    \cline{2-6}
    \multicolumn{1}{c|}{} & FedGKD \cite{yao2021local} & $\CIRCLE$ & $\Circle$ & $\CIRCLE$ & $\CIRCLE$ \\
    \hline
    \multicolumn{1}{c|}{Multiple-objective} & q-FFL \cite{li2019fair} & $\Circle$ & $\CIRCLE$ & $\CIRCLE$ & $\CIRCLE$ \\
    \cline{2-6}
    \multicolumn{1}{c|}{optimization-based} & FedMGDA+ \cite{hu2022federated} & $\Circle$ & $\CIRCLE$ & $\CIRCLE$ & $\CIRCLE$ \\
    \hline
    \multicolumn{2}{c|}{\textbf{FedKF} (ours)} & $\CIRCLE$ & $\CIRCLE$ & $\CIRCLE$ & $\CIRCLE$ \\
\hline
\end{tabular}
}
\end{table*}

To handle data heterogeneity in FL, several previous solutions have been developed.
These solutions can be roughly divided into three categories: 1) model performance-based solutions \cite{li2020federated, zhu2021data, li2021model, yao2021local, lin2020ensemble, luo2021no}, 2) multiple-objective optimization-based solutions \cite{li2019fair, hu2022federated}, and 3) personalized FL solutions \cite{t2020personalized, fallah2020personalized, ammad2019federated}.
For the model performance-based solutions, they merely consider improving the model accuracy in heterogeneous FL, so the fairness issue is not well addressed.
As to the multiple-objective optimization-based solutions, they aim to optimize multiple objectives such as model performance, fairness, robustness, etc.
However, they usually trade off model performance for higher fairness, so the model performance in these solutions is worse than FedAvg.
Personalized FL solutions work on the post-model-training phase (i.e., model adaption phase); thus, they are a complementary approach to the solution proposed in this paper.
Consequently, we aim to design a privacy-preserving FL scheme (that works in the global model training phase) that achieves both high (i.e., over FedAvg) model performance and high fairness for devices in heterogeneous FL.

Observing the limitations of the previous solutions, a question naturally arises: in heterogeneous FL, is it possible to design a privacy-preserving FL scheme (works in the global model training phase) that achieves a better model performance than FedAvg while still keeping fairness as high as possible?
To answer the question, we design \textbf{Fed}erated learning with global-local \textbf{K}nowledge \textbf{F}usion (FedKF) scheme.
FedKF does not need to trade off model performance (to below FedAvg) for fairness while still achieving high model fairness and a better model performance than FedAvg.
Therefore, FedKF yields a positive answer ``Yes" to the above question.
The key idea in FedKF is to let the server return global knowledge to shepherd the local training so that all local models will be regularized towards the global optima, thereby reducing the client model drift issue in each training round.
The global knowledge shepherded local training process is also named the global-local knowledge fusion process in this paper.
In FedKF, two major technical challenges should be well addressed.

The first technical challenge is how to represent global knowledge in each training round.
To address this issue, we design T1 (active-inactive model aggregation technique) to generate a model representing global knowledge.
In each training round, T1 aggregates not only the active clients'  model weights but also the inactive clients' cached model weights obtained in the previous training rounds.
Thus, T1 supports a more precise global knowledge representation.

The second technical challenge is how to use global knowledge to shepherd the local training in a privacy-preserving manner.
To tackle this challenge, we develop T2 (global-local knowledge fusion technique) to enable each local model to learn both global and local knowledge (embedded in each local dataset).
Specifically, we use knowledge distillation to transfer the global knowledge from the T1-aggregated model (teacher model) to each local model (student model) during the local training process.
However, knowledge distillation on a local dataset is hard to accurately transfer global knowledge due to the inconsistency in data distribution between the local and global datasets.
To address this issue, in T2, FedKF lets each client train a local generator to generate imitated samples that follow the distribution of the global dataset.
Thus, FedKF can facilitate knowledge distillation using the generated samples to transfer the global knowledge from the T1-aggregated model to the local model when the local model learns the local knowledge, realizing the global-local knowledge fusion process.

The main contributions of this paper are summarized as follows.

\begin{itemize}

\item We make the first step forward to design a privacy-preserving FL scheme that achieves both high (i.e., over FedAvg) model performance and high fairness for devices in heterogeneous FL.

\item We propose two techniques, T1 (active-inactive model aggregation technique) and T2 (global-local knowledge fusion technique), used in FedKF.
Both T1 and T2 can help to improve model performance and model fairness in heterogeneous FL.

\item We theoretically prove that FedKF can directly turn out to be a good solution to achieve high model performance and high fairness in heterogeneous agnostic FL.
Thus, FedKF has much broader impacts in reality.

\end{itemize}

The remainder of the paper is organized as follows.
Section \ref{related} reviews the related work,
In Section \ref{problem}, we introduce the background knowledge and formal problem statement.
In Section \ref{method}, the detailed FedKF design is presented.
Analysis of FedKF is performed in Section \ref{analysis}.
In Section \ref{experiments}, we evaluate FedKF and report the experimental results.
In Section \ref{conclusion}, some conclusions are drawn.

\section{Related Work}
\label{related}

To handle data heterogeneity in FL, we can develop solutions at two phases of FL: global model training phase and local model adaption phase. 
In the global model training phase, several solutions have been proposed.
These previous solutions can be further classified into two categories: model performance-based solutions and multiple-objective optimization-based solutions.
In the local model adaption phase, personalized FL solutions were developed.
These solutions are analyzed as follows.

\noindent
\textbf{Global Model Training Phase: Model Performance-based Solutions.}
These solutions merely consider improving the model accuracy in heterogeneous FL.
These solutions include FedGen \cite{zhu2021data}, CCVR \cite{luo2021no}, FedDF \cite{lin2020ensemble}, MOON \cite{li2021model}, FedProx \cite{li2020federated}, and FedGKD \cite{yao2021local}.
In FedGen, the server learns a lightweight generator to ensemble knowledge of local models in a data-free manner, then broadcasts it to clients to regularize the local training.
CCVR adjusts the global model using virtual representations sampled from an approximated Gaussian mixture model.
FedDF proposes a model fusion method using ensemble distillation where the server fuses the knowledge from local models and transfers it to the global model by using knowledge distillation on unlabeled proxy data.
FedGen and CCVR require each client to share additional local dataset information with the server.
More specifically, FedGen requires local label count information leaked to the server, while CCVR requires mean and covariance of local features for each class leaked to the server. 
Compared with FedAvg, these two solutions suffer from additional information leakage, resulting in a security level downgrade.
Moreover, compared with FedKF, FedGen needs extra communication overhead.
This is because FedGen needs to train a generator on the server with the label count and broadcast the generator to clients in each training round.
Compared with FedGen, FedKF employs a more accurate knowledge distillation approach. For knowledge distillation, FedGen uses pseudo features with labels generated by the generator directly to train the predictor of the local model.
In FedKF, global knowledge flows from the teacher model to the local model by knowledge distillation with pseudo samples generated by the local generator.
As to FedDF, it assumes there are additional proxy data available on the central server for ensemble distillation.
Such a strong assumption makes them impractical since the proxy data is unavailable in most cases.
FedProx can be viewed as a generalization and re-parametrization of FedAvg.
FedProx adds a proximal term to the local subproblem to restrict the local update closer to the initial (global) model.
MOON utilizes the similarity between model representations to correct the local training, i.e., conducting contrastive learning at the model level.
FedGKD fuses the knowledge from historical global models to guide the local model training where each client learns the global knowledge from past global models via adaptive knowledge distillation techniques.
For FedProx, MOON, and FedGKD, their performances are good on less heterogeneous data, but their performances decrease dramatically as data heterogeneity increases.
Note that none of the solutions in this category are fairness-aware in their initial design.

\noindent
\textbf{Global Model Training Phase: Multiple-objective Optimization-based Solutions.}
These solutions aim to optimize multiple objectives such as model performance, fairness, robustness, etc.
These solutions include q-FFL \cite{li2019fair} and FedMGDA+ \cite{hu2022federated}.
Q-FFL reweights the objective—assigning higher weights to clients with poor performance to encourage a more uniform accuracy distribution across clients in FL.
FedMGDA+ uses multi-objective optimization to obtain a fairer global model and guarantee that the global model converges to Pareto stationary solutions, refraining from sacrificing the performance of any client.
Nevertheless, the fairness gain of q-FFL and FedMGDA+ is obtained by sacrificing model performance, so they cannot achieve a better model performance than FedAvg in theory. 
For example, when setting the reweighting parameter $q=0$ in q-FFL, it reduces to FedAvg.
If $q>0$, the model performance of q-FFL is worse than FedAvg.

\noindent
\textbf{Local Model Adaption Phase: Personalized FL Solutions.}
These solutions leverage various strategies to adapt the global model to
each local dataset.
Some recent solutions include \cite{t2020personalized, fallah2020personalized, ammad2019federated}.
A good survey can be found in \cite{tan2022towards}.
Note that personalized FL solutions can be used together with the solutions in the global model training phase to further enhance the performance.
\section{Preliminaries and Problem Statement}
\label{problem}

In this section, we introduce the background knowledge and formally describe the problem statement.

\subsection{Preliminaries}

\noindent
\textbf{Federated Learning.}
FL is a distributed machine learning setting where a group of clients jointly train a high-quality centralized model without requiring clients to share their local private data \cite{konevcny2016federated}.
Suppose that there are $K$ clients jointly to train an ML model in FL. 
For the $k$-th client, it stores a local dataset $\mathcal{D}_{k}$.
FL aims to learning a global model weight $w$ over the global dataset $\mathcal{D} = \cup \{ \mathcal{D}_{k}\}_{k=1}^{K}$.
Accordingly, the objective of FL is to solve the following optimization problem \cite{mcmahan2017communication}:
\begin{equation}
    \min_{w} {f(w)} = \sum _{k=1}^{K} \frac{|\mathcal{D}_{k}|}{|\mathcal{D}|} F_{k}(w),
\end{equation}
where $F_{k}(w) = \mathbb{E}_{(x, y)\sim \mathcal{D}_{k}} [\ell (w;x,y)]$ represents the local objective function at the $k$-th client.

\noindent
\textbf{Knowledge Distillation.} 
Knowledge distillation (KD) technique enables a student model to learn from one or multiple teacher models \cite{hinton2015distilling, fukuda2017efficient}.
KD supports the student model compression while enabling the student model to inherit knowledge distilled from teacher(s).
The classic KD techniques (e.g., \cite{hinton2015distilling, lin2020ensemble}) require a proxy dataset during the distillation. 
To eliminate the requirement for the proxy dataset, data-free KD is proposed \cite{lopes2017data, chen2019data, fang2019data}.
A popular solution for data-free KD is to use the idea of generative adversarial networks (GANs) \cite{chen2019data}. 
A generator is trained to produce imitated training data (to replace the original training dataset) used for KD.

\subsection{Problem Statement}
We consider the local datasets unbalanced and non-IID (i.e., heterogeneous). 
Next, we define the following three metrics to facilitate our problem description. 
\newtheorem{definition}{Definition}

\begin{definition}
\label{def21}
(\textbf{Average Model Performance}) 
The average model performance (AMP) of model $w$ (denoted as $AMP_w$) is defined as the average test accuracy of model $w$ on the $K$ clients.
Let $a_{w}^{k} (k=1, \ldots, K)$ represent the test accuracy of model $w$ on $k$-th client's local test dataset.
$AMP_w$ can also be computed as  $AMP_w = \sum_{k=1}^{K}\frac{|\mathcal{D}_{k}|}{|\mathcal{D}|}a_{w}^{k}$.

\end{definition}

\begin{definition}
\label{def22}
 (\textbf{Fairness Metric}) The fairness metric (FM) of model $w$ (denoted as $FM_w$) is defined as $FM_w=\mathrm{Var}(a_{w}^{1}, \ldots, a_{w}^{K}),$ where $\mathrm{Var}$ denotes the variance. 
It is given by 
$\mathrm{Var}=\frac{1}{K} \sum_{k=1}^{K}(a_{w}^{k}-\overline{a}_{w})^2$ and $\overline{a}_{w}=\frac{1}{K}\sum_{k=1}^{K}a_{w}^{k}$.
A smaller $FM_w$ indicates a fairer model $w$. 

\end{definition}

\begin{definition}
\label{def23}
(\textbf{Worst-case Local Performance}) The worst-case local performance (WLP) of model $w$ (denoted as $WLP_w$) is defined as  $WLP_w=\min\{ a_{w}^{1}, \ldots, a_{w}^{K}\}$.

\end{definition}

\noindent
\textbf{WLP as Joint Performance Metric.}
According to Definitions (\ref{def21})-(\ref{def23}), the WLP metric can be treated as measuring the joint performance of AMP and FM. 
On the one hand, given the fixed AMP, a larger WLP tends to indicate a more uniform distribution of local performance (i.e., smaller FM) across clients. 
On the other hand, given the fixed FM, a larger WLP tends to imply a model with a higher AMP.

\begin{definition}
\label{def:privacy}
(\textbf{Privacy-Preserving}) We say an FL scheme is privacy-preserving if it follows the same security principle as FedAvg: for each client, only its model weight can be sent to other entities (e.g., server), and no information about local data can be shared directly.
\end{definition}

\noindent
\textbf{Design Goals.}
In this paper, we aim to design a privacy-preserving FL scheme FedKF that can increase AMP while reducing FM. 
Thus, FedKF can achieve high global performance and fairness simultaneously. 
If both AMP and FM are jointly considered, FedKF should keep WLP as large as possible.
To better explain why WLP can be treated as a joint performance metric, a numerical example is provided. 
Suppose that there are three models $w_1$, $w_2$, and $w_3$ trained over three clients via FL. 
Their parameter configurations are shown in Table \ref{tab:num}.
For $w_1$ and $w_2$, it holds that $AMP_{w_1}=AMP_{w_2}$. 
Since $WLP_{w_1}<WLP_{w_2}$, we have $FM_{w_1}>FM_{w_2}$.
For $w_1$ and $w_3$, it holds that $FM_{w_1}=FM_{w_3}$.
Since $WLP_{w_1}<WLP_{w_3}$, we have $AMP_{w_1}<AMP_{w_3}$.

\begin{table}[ht]
  \centering
  \caption{An numerical example.}
  \resizebox{1\columnwidth}{!}
  {
  \renewcommand{\arraystretch}{1.5}
  \begin{tabular}{c|c|c|c|c|c|c}
    \hline
      & Client 1 & Client 2 & Client 3 & AMP & FM & WLP \\
    \hline
    \hline
    $w_{1}$ & $a_{w_{1}}^{1} \!= \! 0.6$ & $a_{w_{1}}^{2} \!=\! 0.7$ & $a_{w_{1}}^{3} \!=\! 0.8$ & 0.7 & 0.00667 & 0.6\\
    \hline
    $w_{2}$ & $a_{w_{2}}^{1} \!=\! 0.65$ & $a_{w_{2}}^{2} \!= \!0.65$ & $a_{w_{2}}^{3} \!= \!0.8$ & 0.7 & 0.005 & 0.65 \\
    \hline
    $w_{3}$ & $a_{w_{3}}^{1} \!= \!0.7$ & $a_{w_{3}}^{2} \!=\! 0.8$ & $a_{w_{3}}^{3} \!=\! 0.9$ & 0.8 & 0.00667 & 0.7 \\
    \hline
  \end{tabular}
  }
  \label{tab:num}
\end{table}

\noindent
\textbf{Notations.}
To improve the readability of this paper, we summarize some frequently used notations in Table \ref{tb::not}.

\begin{table}[h]
  \centering
  \caption{Notations.}
  {
  \renewcommand{\arraystretch}{1.2}
  \begin{tabular}{l|l}
    \hline
    Notations & Meanings \\
    \hline
    \hline
    $K$ & the number of overall clients \\
    $m$ & the number of active clients \\
    $C$ & selection rate of active clients \\
    $T$ & the number of global rounds \\
    $E$ & the number of local epochs \\
    $B$ & local batch size \\
    $\alpha$ & concentration parameter of the Dirichlet distribution \\
    $\beta$ & the learning rate for training local generator \\
    $\eta$ & the learning rate for training local model \\
    $\lambda _{1}$ & the coefficient of one-hot loss for training generator \\
    $\lambda _{2}$ & the coefficient of activation loss for training generator \\
    $\gamma$ & the coefficient of KL loss for training local model \\
    $\mathcal{D}_{k}$ & the $k$-th client's local dataset \\
    $\mathcal{D}$ & global dataset \\
    AMP & Average Model Performance \\
    FM & Fairness Metric \\
    WLP & Worst-case Local Performance \\
    ACA model & active clients aggregated model \\
    OCA model & overall clients aggregated model \\
    $w^{t}$ & the ACA model in the $t$-th round \\
    $\hat{w}^{t}$ & the OCA model in the $t$-th round \\
    $w_{k}^{t}$ & the $k$-th client's local model in the $t$-th round \\
    $\theta_{k}$ & the $k$-th client's local generator \\
    \hline
  \end{tabular}
}
  \label{tb::not}
\end{table}
\begin{figure*}[ht]
  \centering
  \includegraphics[width=0.8\textwidth]{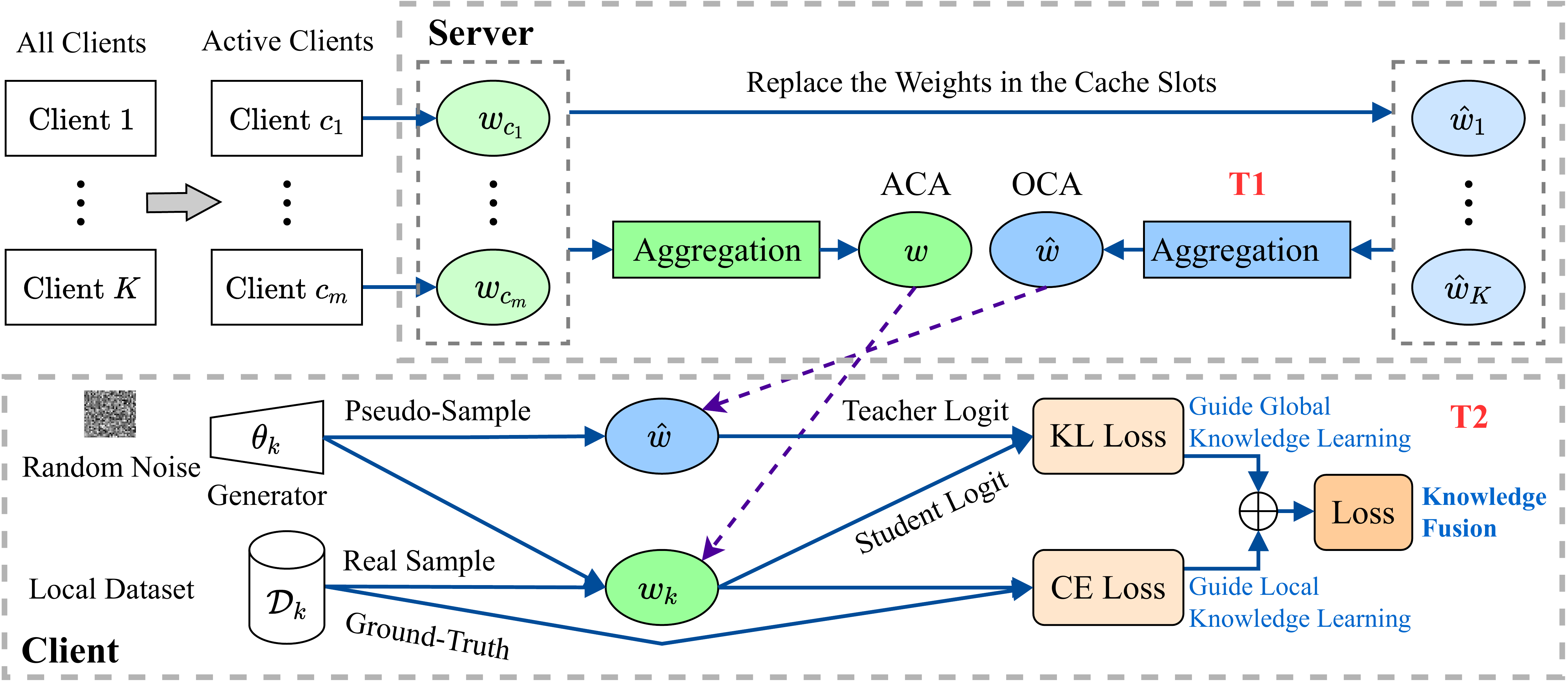}
  \caption{FedKF Overview. Two key techniques (i.e., T1 and T2) are developed in FedKF. In T1, we develop an active-inactive model aggregation technique to generate an OCA model that represents the global knowledge precisely. In T2, we develop the global-local knowledge fusion technique to enable the local model to learn both the global knowledge (embedded in the teacher model, i.e., the OCA model) and the local knowledge (embedded in the local dataset).}
  \label{Fig::FSpace}
\end{figure*}

\section{FedKF Design}
\label{method}

In this section, we first overview FedKF and two developed key techniques.
Then, we introduce the T2 (global-local knowledge fusion technique) used in FedKF.

\subsection{FedKF \& Key Techniques Overview}

\label{subsect:overview}

\noindent
\textbf{FedKF.}
An overview of FedKF is illustrated in Fig. \ref{Fig::FSpace}.
In FedKF, the server maintains $K$ different cache slots for storing the latest local models.
In each training round, only the selected active clients need to upload their local models to the server.
Thus, the $k$-th cache slot stores the local model uploaded from $k$-th client in the most recent training round when $k$-th client is selected to be active.
Informally, FedKF can be described as follows.

\begin{itemize}
\item \emph{Step 0}: In the last step of the training round $t-1$, the server aggregates all active clients' uploaded local models to get the active clients aggregated (ACA) model.
Meanwhile, FedKF aggregates both active and inactive clients' cached models in the cache slots to get the overall clients aggregated (OCA) model.
In this step, active and inactive refer to the clients' state in round $t-1$.

\noindent
\item \emph{Step 1}: In the training round $t$, a portion of clients are selected as active clients.
Let $\{c_{1}, \ldots, c_{m}\}$ denote the IDs of the selected clients.
Let $C$ represent the selection rate.
It follows that $m= C \cdot K$.
Then, the server broadcasts the ACA and OCA models to all active clients.

\noindent
\item \emph{Step 2}: On receipt of the two models from the server, each active client treats the ACA model as the local model $w_{k}$ and treats the OCA model as the teacher model $\hat{w}$.

\noindent
\item  \emph{Step 3}: FedKF employs the data-free KD technique to distill the knowledge of the teacher model to the local model.
In data-free KD, a generator is trained to generate pseudo-samples to facilitate knowledge transfer.
Meanwhile, the local dataset of each active client is used to train the local model.
Therefore, both the global knowledge (embedded in the teacher model) and local knowledge (embedded in the local dataset) are fused and transferred to the local model.

\noindent
\item  \emph{Step 4}: After global-local knowledge fusion, all active clients upload their local models to the server.
Each client's local model serves as the latest local model.

\noindent
\item  \emph{Step 5}: Based on the received active clients' local models, the server updates the weights in the corresponding cache slots.
The inactive clients' cache slots remain the same.
Then, the server re-computes the ACA model and the OCA model.
Next, if the model is well trained, then terminate the training process; otherwise, go to Step 1.

\end{itemize}

When the training is finished, either the ACA model or the OCA model serves as the final model to be used.
The detailed FedKF training algorithm is shown in Algorithm \ref{alg::train}.

Note that FedKF does not use the server-side generator training approach because the server-side generator training approach achieves much worse model performance (i.e., AMP) than our client-side generator training approach. The main reason is that for the client-side generator training approach, the generator is trained along with the KD process (i.e., the generator training and KD are trained synchronously), so the generator can generate more diversified samples used in KD. In contrast, in the server-side approach, a stationary generator (i.e., well-trained) is used to generate the samples used in KD, so the generated samples are less diversified, resulting in worse model performance.

\noindent
\textbf{Two Key Techniques.}
In FedKF, two key techniques are developed.
\begin{itemize}
\item \textbf{T1} (active-inactive model aggregation technique): We develop an active-inactive model aggregation technique to generate an OCA model that represents the global knowledge precisely.

\item \textbf{T2} (global-local knowledge fusion technique): We develop the global-local knowledge fusion technique to enable the local model to learn both the global knowledge (embedded in the teacher model) and the local knowledge (embedded in the local dataset).

\end{itemize}
For most previous solutions (e.g., FedAvg), only active clients' model weights are aggregated to generate the global model in each round.
In contrast, in T1, both active clients' model weights and inactive clients' cached model weights are aggregated to represent the global knowledge.
Hence, T1 supports a more precise global knowledge representation.
Note that the authors in \cite{li2019convergence} discuss a solution that uses the global weight to represent the weight of inactive clients in aggregation; it is a coarse global knowledge representation method.
T1 is a simple yet precise approach to generating the global model.
It is orthogonal to many previous solutions (e.g., FedAvg, FedProx), so T1 can also be used in these solutions to improve their performance.
In the following sections, T2 is elaborated.

\subsection{Data-Free Knowledge Distillation}
\label{sub:data-free}

Knowledge distillation (KD) technique enables a student model to learn from one or multiple teacher models \cite{hinton2015distilling, fukuda2017efficient}.
In FedKF, to distill knowledge from the global model (teacher model) to a local model (student model), the data used to train the global model is usually required.
However, in FL, data exchange is prohibited due to security concerns (see Definition \ref{def:privacy}).
Accordingly, FedKF employs the idea of data-free KD \cite{lopes2017data, chen2019data, fang2019data} to eliminate the requirement of the proxy data on the client side.

In data-free KD, a generator can be trained and then used for generating the imitated training samples.
The imitated training samples can be used to transfer knowledge from the teacher model to the student model.
Note that the generated imitated training samples do not need to be distributed very similarly to the real training samples.
The only requirement is that the generator training samples can be used to facilitate knowledge transfer.
Hence, the requirement for the generator in data-free KD and the generator in a traditional GAN is different.

\subsection{Loss Functions for Training Local Generator}
\label{subsect:TG}

Inspired by \cite{chen2019data}, we design the following loss functions used in the local generator training.
To facilitate the description, we first describe the following parameters.
Let $g(\theta_{k}; \cdot)$ be the output of the $k$-th client's local generator parameterized by $\theta_{k}$.
Let $f(\hat{w}; \cdot)$ and $h(\hat{w}; \cdot)$ be the feature vector output and the probability vector output of the teacher model parameterized by $\hat{w}$, respectively.
On input a random noise vector $z \sim \mathcal{N}(0, I)$, the generator outputs pseudo-sample $\hat{x}$ with $\hat{x} = g(\theta_{k}; z)$.
On input $\hat{x}$, the teacher model can output probability vector $\hat{\mathbf{p}}$ with $\hat{\mathbf{p}} = h(\hat{w}; \hat{x})$.

Based on the above definitions, the loss functions for training the local generator are introduced as follows.

\noindent
\textbf{One-Hot Loss Function.}
The pseudo-sample $\hat{x}$ is expected to be classified into one particular category concerned by the teacher model with a higher probability.
Then, pseudo-label $\hat{y}$ is calculated by $\hat{y} = \mathrm{arg} \max \hat{\mathbf{p}}$.
The one-hot loss function is defined as
\begin{equation}
    \mathcal{L}_{OH} = \mathbb{E}_{z \sim \mathcal{N}(0, I)} \left[ \mathrm{CE}\left(h\left(\hat{w}; g\left( \theta_{k}; z \right) \right), \hat{y} \right) \right],
\end{equation}
where $\mathrm{CE}$ is cross entropy. If $\mathcal{L}_{OH}$ is minimized, then a generated sample can be classified into one specific class with a significantly high probability.
This phenomenon occurs when real samples are used for training.

\noindent
\textbf{Information Entropy Loss Function.}
In order to force the generator to generate samples covering all classes, the information entropy loss is used to measure the uniformity of the class distribution.
Specifically, given a probability vector $\mathbf{p} = (p_{1}, p_{2}, \ldots, p_{d})$, the information entropy of $\mathbf{p}$ is calculated by $\mathrm{IE}(\mathbf{p}) = - \sum_{i=1}^{d}p_{i}\log p_{i}$.
The information entropy loss can be defined as

\begin{equation}
    \mathcal{L}_{IE} = - \mathrm{IE} (\mathbb{E}_{z \sim \mathcal{N}(0, I)} \left[ h\left(\hat{w}; g\left( \theta_{k}; z \right) \right) \right]).
\end{equation}
When $\mathcal{L}_{IE}$ moves to the minimum, the generator tends to generate samples for each class with roughly the same probability.
Thus, minimizing the information entropy loss can result in a training sample set in which the number of samples for each class is roughly the same.

\noindent
\textbf{Activation Loss Function.}
It is observed that the real training sample's feature vector tends to receive a higher activation value.
Thus, the activation loss function is defined as

\begin{equation}
    \mathcal{L}_{A} = -\mathbb{E}_{z \sim \mathcal{N}(0, I)} [\left \| f\left(\hat{w}; g\left( \theta_{k}; z \right) \right) \right \|_{1}],
\end{equation}
where $\left \| \cdot \right \| _{1}$ is the $l_{1}$ norm.

\noindent
\textbf{Total Loss Function.}
By taking the above three loss functions into consideration, the total loss function for the generator training is given by
\begin{equation}\label{eq:gloss}
    \mathcal{L}_{G} = \mathcal{L}_{IE} + \lambda_{1} \mathcal{L}_{OH} + \lambda_{2} \mathcal{L}_{A},
\end{equation}
where $\lambda_{1}$ and $\lambda_{2}$ are hyper parameters for balancing the three loss functions.

\subsection{Performance of Trained Generator}

In KD, theoretically, we can use randomly generated training samples to train the student model to mimic the behavior of the teacher model.
Note that a randomly generated training sample can be fed to the teacher model to get its label.
Then, the properly labeled samples can be used for the student model training.
However, this approach has low efficiency and usually cannot achieve high KD accuracy.
In FedKF, for the generator used in KD, if the generator can generate samples that are distributed relatively close to the real-world training samples, the KD process can be finished with high accuracy.
The generated samples do not need to be as accurate as some other applications (e.g., the generator required by deepfake \cite{westerlund2019emergence}).

In what follows, we conduct some experiments to demonstrate the performance of the trained generator used in data-free KD. In the experiments, we train a teacher model on centralized real-world training data.
Then, we use the teacher model to train a generator with the loss function (as defined in Eq. \ref{eq:gloss}) in a data-free manner.
Meanwhile, we use the generated samples to train a student model via KD.

The experiment settings are briefly introduced as follows.
The optimizers used in training the teacher model, student model, and generator are SGD, SGD, and Adam, respectively.
As for training the teacher and student model, the learning rate is set to 0.01 for LeNet-5 \cite{lecun1989backpropagation} and 0.1 for ResNet-8\cite{he2016deep}.
For training the generator, the learning rate is set to 0.2, 0.02, and 0.02 on the EMNIST\cite{cohen2017emnist}, CIFAR-10, and CIFAR-100\cite{cohen2017emnist} dataset, respectively.
For hyper parameters $\lambda_{1}$ and $\lambda_{2}$ in training the generator, we set them to \{0.2, 0.02\} on the EMNIST dataset, \{0.01, 0.002\} on the CIFAR-10 dataset, and \{0.01, 0.002\} on the CIFAR-100 dataset.

Fig. \ref{fig:reals} and \ref{fig:fakes} show the visualization results of averaged images on the EMNIST dataset and the generated dataset (using the local generator), respectively.
Although the generated
images are not very similar to the real images used in training, but it is sufficient for achieving good performance in KD (as demonstrated by the experimental results below).

\begin{figure}[h]
\centering
\subfigure[]{
    \centering
    \includegraphics[width=0.072\linewidth]{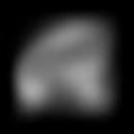}
}
\hspace{-4.5mm}
\subfigure[]{
    \centering
    \includegraphics[width=0.072\linewidth]{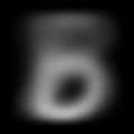}
}
\hspace{-4.5mm}
\subfigure[]{
    \centering
    \includegraphics[width=0.072\linewidth]{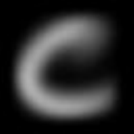}
}
\hspace{-4.5mm}
\subfigure[]{
    \centering
    \includegraphics[width=0.072\linewidth]{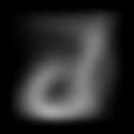}
}
\hspace{-4.5mm}
\subfigure[]{
    \centering
    \includegraphics[width=0.072\linewidth]{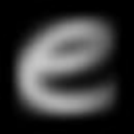}
}
\hspace{-4.5mm}
\subfigure[]{
    \centering
    \includegraphics[width=0.072\linewidth]{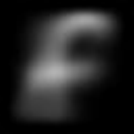}
}
\hspace{-4.5mm}
\subfigure[]{
    \centering
    \includegraphics[width=0.072\linewidth]{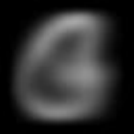}
}
\hspace{-4.5mm}
\subfigure[]{
    \centering
    \includegraphics[width=0.072\linewidth]{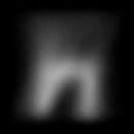}
}
\hspace{-4.5mm}
\subfigure[]{
    \centering
    \includegraphics[width=0.072\linewidth]{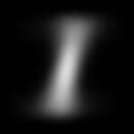}
}
\hspace{-4.5mm}
\subfigure[]{
    \centering
    \includegraphics[width=0.072\linewidth]{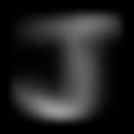}
}
\hspace{-4.5mm}
\subfigure[]{
    \centering
    \includegraphics[width=0.072\linewidth]{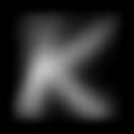}
}
\hspace{-4.5mm}
\subfigure[]{
    \centering
    \includegraphics[width=0.072\linewidth]{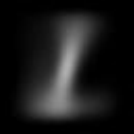}
}
\hspace{-4.5mm}
\subfigure[]{
    \centering
    \includegraphics[width=0.072\linewidth]{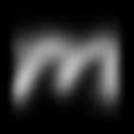}
}
\hspace{-4.5mm}
\subfigure[]{
    \centering
    \includegraphics[width=0.072\linewidth]{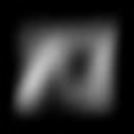}
}
\hspace{-4.5mm}
\subfigure[]{
    \centering
    \includegraphics[width=0.072\linewidth]{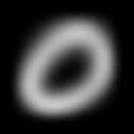}
}
\hspace{-4.5mm}
\subfigure[]{
    \centering
    \includegraphics[width=0.072\linewidth]{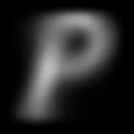}
}
\hspace{-4.5mm}
\subfigure[]{
    \centering
    \includegraphics[width=0.072\linewidth]{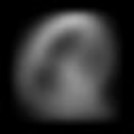}
}
\hspace{-4.5mm}
\subfigure[]{
    \centering
    \includegraphics[width=0.072\linewidth]{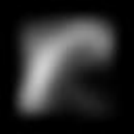}
}
\hspace{-4.5mm}
\subfigure[]{
    \centering
    \includegraphics[width=0.072\linewidth]{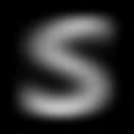}
}
\hspace{-4.5mm}
\subfigure[]{
    \centering
    \includegraphics[width=0.072\linewidth]{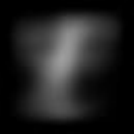}
}
\hspace{-4.5mm}
\subfigure[]{
    \centering
    \includegraphics[width=0.072\linewidth]{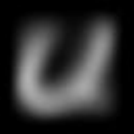}
}
\hspace{-4.5mm}
\subfigure[]{
    \centering
    \includegraphics[width=0.072\linewidth]{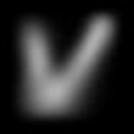}
}
\hspace{-4.5mm}
\subfigure[]{
    \centering
    \includegraphics[width=0.072\linewidth]{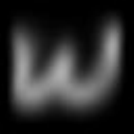}
}
\hspace{-4.5mm}
\subfigure[]{
    \centering
    \includegraphics[width=0.072\linewidth]{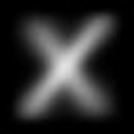}
}
\hspace{-4.5mm}
\subfigure[]{
    \centering
    \includegraphics[width=0.072\linewidth]{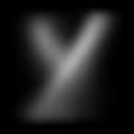}
}
\hspace{-4.5mm}
\subfigure[]{
    \centering
    \includegraphics[width=0.072\linewidth]{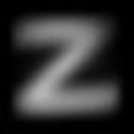}
}
\caption{Visualization of the averaged image in each category (from a to z) on the EMNIST dataset.}
\label{fig:reals}
\end{figure}

\begin{figure}[ht]
\centering
\subfigure[]{
    \centering
    \includegraphics[width=0.072\linewidth]{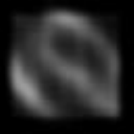}
}
\hspace{-4.5mm}
\subfigure[]{
    \centering
    \includegraphics[width=0.072\linewidth]{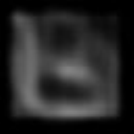}
}
\hspace{-4.5mm}
\subfigure[]{
    \centering
    \includegraphics[width=0.072\linewidth]{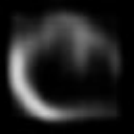}
}
\hspace{-4.5mm}
\subfigure[]{
    \centering
    \includegraphics[width=0.072\linewidth]{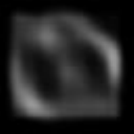}
}
\hspace{-4.5mm}
\subfigure[]{
    \centering
    \includegraphics[width=0.072\linewidth]{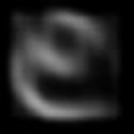}
}
\hspace{-4.5mm}
\subfigure[]{
    \centering
    \includegraphics[width=0.072\linewidth]{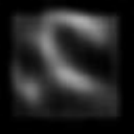}
}
\hspace{-4.5mm}
\subfigure[]{
    \centering
    \includegraphics[width=0.072\linewidth]{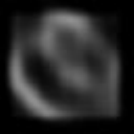}
}
\hspace{-4.5mm}
\subfigure[]{
    \centering
    \includegraphics[width=0.072\linewidth]{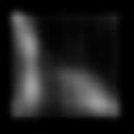}
}
\hspace{-4.5mm}
\subfigure[]{
    \centering
    \includegraphics[width=0.072\linewidth]{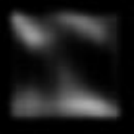}
}
\hspace{-4.5mm}
\subfigure[]{
    \centering
    \includegraphics[width=0.072\linewidth]{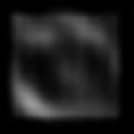}
}
\hspace{-4.5mm}
\subfigure[]{
    \centering
    \includegraphics[width=0.072\linewidth]{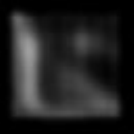}
}
\hspace{-4.5mm}
\subfigure[]{
    \centering
    \includegraphics[width=0.072\linewidth]{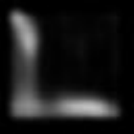}
}
\hspace{-4.5mm}
\subfigure[]{
    \centering
    \includegraphics[width=0.072\linewidth]{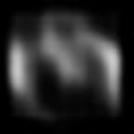}
}
\hspace{-4.5mm}
\subfigure[]{
    \centering
    \includegraphics[width=0.072\linewidth]{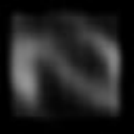}
}
\hspace{-4.5mm}
\subfigure[]{
    \centering
    \includegraphics[width=0.072\linewidth]{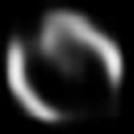}
}
\hspace{-4.5mm}
\subfigure[]{
    \centering
    \includegraphics[width=0.072\linewidth]{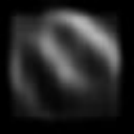}
}
\hspace{-4.5mm}
\subfigure[]{
    \centering
    \includegraphics[width=0.072\linewidth]{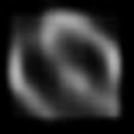}
}
\hspace{-4.5mm}
\subfigure[]{
    \centering
    \includegraphics[width=0.072\linewidth]{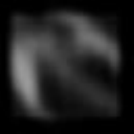}
}
\hspace{-4.5mm}
\subfigure[]{
    \centering
    \includegraphics[width=0.072\linewidth]{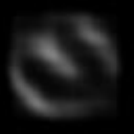}
}
\hspace{-4.5mm}
\subfigure[]{
    \centering
    \includegraphics[width=0.072\linewidth]{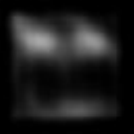}
}
\hspace{-4.5mm}
\subfigure[]{
    \centering
    \includegraphics[width=0.072\linewidth]{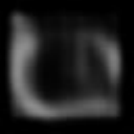}
}
\hspace{-4.5mm}
\subfigure[]{
    \centering
    \includegraphics[width=0.072\linewidth]{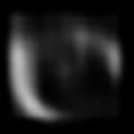}
}
\hspace{-4.5mm}
\subfigure[]{
    \centering
    \includegraphics[width=0.072\linewidth]{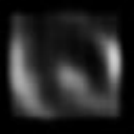}
}
\hspace{-4.5mm}
\subfigure[]{
    \centering
    \includegraphics[width=0.072\linewidth]{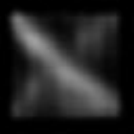}
}
\hspace{-4.5mm}
\subfigure[]{
    \centering
    \includegraphics[width=0.072\linewidth]{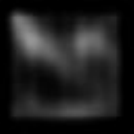}
}
\hspace{-4.5mm}
\subfigure[]{
    \centering
    \includegraphics[width=0.072\linewidth]{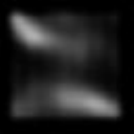}
}
\caption{Visualization of the averaged image in each category (from a to z) on the generated dataset (using the local generator).}
\label{fig:fakes}
\end{figure}

Table \ref{dfkd:acc} reports the performance of the teacher and student models on the different datasets.
The teacher models achieve 93.35\%, 90.78\%, and 67.14\% accuracies on the EMNIST, CIFAR-10, and CIFAR-100 datasets, respectively.
The student models using T2 obtain 92.25\%, 89.05\%, and 63.45\% accuracies without any real-world training data.
The performance of the student model is just slightly lower than the teacher model.

\begin{table}[h]
\centering
\caption{Performance of the teacher and student models in T2 on the different datasets.}
\label{dfkd:acc}
{
\renewcommand{\arraystretch}{1.4}
\begin{tabular}{c|cc|cc}
\hline
\multirow{2}{*}{Datasets} & \multicolumn{2}{c|}{Teacher} & \multicolumn{2}{c}{Student} \\
\cline{2-5}
                          & Model     & Accuracy (\%)   & Model       & Accuracy (\%) \\
\hline
\hline
EMNIST                    & LeNet-5   & 93.35          & LeNet-5      & 92.25 \\
\hline
CIFAR-10                  & ResNet-8  & 90.78          & ResNet-8     & 89.05 \\
\hline
CIFAR-100                 & ResNet-8  & 67.14          & ResNet-8     & 63.45 \\
\hline
\end{tabular}
}
\end{table}

\subsection{Loss Functions for Training Local Model}
\label{subsect:TS}

The local model is trained by T2 (global-local knowledge fusion technique).
The loss functions are introduced as follows.

\noindent
\textbf{KL Loss Function.}
FedKF allows each client to use the imitated training samples generated by the generator to distill the global knowledge from the teacher model to the local model.
Meanwhile, the local model learns the local knowledge from the local dataset.

Let $h(w_{k}; \cdot)$ be the probability vector output of the $k$-th client's local model parameterized by $w_{k}$.
We define the knowledge distillation loss as
\begin{equation}
\small
    \mathcal{L}_{KL} = \mathbb{E}_{z \sim \mathcal{N}(0, I)} \left[ \mathrm{KL} \left( h(\hat{w}; g(\theta_{k}; z)) \parallel h(w_{k}; g(\theta_{k}; z)) \right) \right],
\end{equation}
where $\mathrm{KL}$ stands for Kullback–Leibler divergence \cite{kullback1951information}.
When minimizing $\mathcal{L}_{KL}$, the local model is moving closer to the teacher model (i.e., learning the global knowledge).

\noindent
\textbf{Cross Entropy Loss Function.}
We define the loss function over the local dataset $\mathcal{D}_{k}$ as
\begin{equation}
    \mathcal{L}_{CE} = \mathbb{E}_{(x, y) \sim \mathcal{D}_{k}} \left[ \mathrm{CE} \left( h(w_{k}; x), y \right) \right].
\end{equation}
When minimizing $\mathcal{L}_{CE}$, the local model is learning the local knowledge (embedded in the local dataset).

\noindent
\textbf{Total Loss Function.}
The total loss function for global-local knowledge fusion is given as
\begin{equation}
    \mathcal{L} = \mathcal{L}_{CE} + \gamma \mathcal{L}_{KL},
\end{equation}
where $\gamma$ is a hyperparameter for balancing the two loss functions.
When minimizing  $\mathcal{L}$, the global-local knowledge is fused to the local model.

\begin{algorithm}[h]
\caption{FedKF Training.}\label{alg::train}
\SetKwInput{KwInput}{Input}                
\SetKwInput{KwOutput}{Output}              
\DontPrintSemicolon
\KwInput{$K$, $C$, $T$, $E$, $B$, $\beta$, $\eta$, $\theta$, $w^{0}$.}
\KwOutput{$w^{T}$, $\hat{w}^{T}$.}

\textbf{Server executes:}\;
{
    $\hat{w}^{0} \gets w^{0}$, $\hat{w}_{k} \gets w^{0}$, $\theta_{k} \gets \theta$ ($k=1,\ldots, K$).\;
    \For{round $t = 1, 2, \ldots, T$}{
        $\mathcal{S}_{t} \gets$ (random set of $m = C \cdot K$ active clients).\;
        \For{each client $k \in \mathcal{S}_{t}$ \textbf{in parallel}}{
            $w_{k}^{t} \gets$ \textbf{ClientUpdate}$(k, w^{t-1}, \hat{w}^{t-1})$.\;
            $\hat{w}_{k} \gets w_{k}^{t}$. // Update the model in the $k$-th cache slot.\;
        }
        $w^{t} \gets \frac{1}{\sum_{k \in \mathcal{S}_{t}} |\mathcal{D}_{k}| }\sum_{k \in \mathcal{S}_{t}} |\mathcal{D}_{k}| w_{k}^{t}$.\;
        $\hat{w}^{t} \gets \frac{1}{\sum_{k=1}^{K} |\mathcal{D}_{k}|}\sum_{k=1}^{K} |\mathcal{D}_{k}| \hat{w}_{k}$.\;
    }
}
\;
\textbf{ClientUpdate {$(k, w, \hat{w})$}:}\;
{
    $w_{k} \gets w$.\;
    $\mathcal{B} \gets$ (split $\mathcal{D}_{k}$ into batches of size $B$).\;
    \For{each local epoch $i = 1, 2, \ldots, E$}{
        \For{each batch $b \in \mathcal{B}$} {
            $\theta_{k} \gets \theta_{k} - \beta \cdot \nabla \mathcal{L}_{G}(\theta_{k})$.  // Update generator via minimizing $\mathcal{L}_{G}$.\;
            $w_{k} \gets w_{k} - \eta \cdot \nabla \mathcal{L}(w_{k})$. // Update local model via minimizing $\mathcal{L}$.\;
        }
    }
    \KwRet $w_{k}$.\;
}
\end{algorithm}

\subsection{FedKF Training Algorithm}
\label{sub:train1}

The detailed FedKF training algorithm is shown in Algorithm \ref{alg::train}.
The FedKF training algorithm requires $K$ (the number of all clients), $C$ (selection rate), $T$ (the number of communication rounds), $E$ (the number of local training epochs), $B$ (local batch size), $\beta$ (the learning rate for training local generator), $\eta$ (the learning rate for training local model), $\theta$ (the initial generator weight) and $w^{0}$ (the initial ACA model weight) as inputs and returns $w^{T}$ (the final ACA model) and $\hat{w}^{T}$ (the final OCA model) as outputs.
In line 2, the server initializes the OCA model $\hat{w}^{0}$ and all the models $\{ \hat{w}_{k} \}_{k=1}^{K}$ in the cache slots with $w^{0}$ and all clients initialize their local generators $\{ \theta_{k} \}_{k=1}^{K}$ with $\theta$.
In line 4, the server uniformly selects $m$ clients as active ones with $m = C \cdot K$ at random.
In line 6, each active client executes \textbf{ClientUpdate} and uploads the latest local model to the server.
In line 7, the server replaces models in the cache slots with the latest models uploaded from active clients in the current round.
In line 9, the server aggregates all the latest models uploaded from active clients in the current round and gets the updated ACA model.
In line 10, the server aggregates all models in the cache slots and gets the updated OCA model.

\section{FedKF Analysis}
\label{analysis}

In this section, we analyze FedKF from four aspects: why high AMP, why high fairness, why privacy-preserving, and its relationship with agnostic FL.

\subsection{Why High Average Model Performance}

There are two techniques contributing to the high AMP of FedKF.

\noindent
\textbf{T1 Helps to Improve AMP.} 
On the server side, most previous solutions use the active clients aggregated (ACA) model as the global model that is used for inference.
In contrast, our solution FedKF (OCA) uses the overall clients aggregated (OCA) model as the global model.
Since the ACA model only aggregates a small portion of clients, it may lead to a large AMP degradation when clients' local datasets are non-IID.
In contrast, when T1 is used to generate the global model on the server side, it aggregates a model that contains more precise global knowledge learned during FL training.
Thus, T1 can significantly increase the AMP of FedKF.

\noindent
\textbf{T2 Helps to improve AMP.} 
In FedAvg, the AMP degradation is caused by the client model drift issue when training on heterogeneous data.
To improve AMP, FedKF uses T2 to address the client model drift issue.
On the client side, when performing the local training, each client learns the global knowledge simultaneously. 
T2 can regularize the local training by jointly considering both the global and local knowledge. 
It can avoid the local model overfitting towards the local dataset.
Thus, the client model drift issue is alleviated, and the AMP of FedKF is boosted.

\subsection{Why High Fairness}

There are two techniques contributing to the high fairness of FedKF.

\noindent
\textbf{T1 Helps to Improve Fairness.} 
On the server side, most previous solutions use the ACA model as the global model, while our solution FedKF (OCA) uses the OCA model as the global model.
Because the ACA model only aggregates a small portion of clients, it may generate a model that is biased towards only the active clients. 
Thus, the ACA model has poor fairness when data is heterogeneous.
On the contrary, if T1 is used to generate the global model, both the inactive and the active clients are taken into consideration, leading to a fairer model.

\noindent
\textbf{T2 Helps to Improve Fairness.}
In FedAvg, the local model is trained only on the local dataset, so the local model could be overfitted on the local dataset. 
It leads to different degrees of overfitting on different clients when their local datasets are non-IID. 
Hence, the AMP variance could be very large, and the model fairness could be low. 
In contrast, T2 can be used to avoid the local model overfitting towards the local dataset since both global knowledge (embedded in the teacher model) and local knowledge (embedded in the local dataset) are fused into the local model. 
Therefore, T2 can help FedKF to achieve higher model fairness.

\subsection{Why Privacy-Preserving}
In each FedKF training round, there are two information flows exchanged between the server and each client. 
First, the server needs to send two models (i.e., the ACA and OCA models) to each client.
Second, each client needs to send the updated local model to the server after the global-local knowledge fusion. 
Hence, no information about the local data is shared directly.
According to Definition \ref{def:privacy}, FedKF is privacy-preserving.

\subsection{Relationship with Agnostic FL}
\label{app:AFL}

The traditional FL is to optimize the model on the global distribution. 
In practice, the target distribution can be very different from the global distribution.
To improve the applicability of FL, agnostic federated learning (AFL) is proposed \cite{mohri2019agnostic}.
AFL aims to optimize the model performance on any possible target distribution formed by a mixture of client distributions.
In other words, AFL has better domain generalization capability. 
Therefore, it captures more use cases and significantly expands the applicability of FL. 

The mathematical description of AFL is presented as follows. 
Let $Dis_{k}$ denote the local data distribution of $k$-th client.
The global distribution $U$ is denoted as $U= \sum_{k=1}^{K} \frac{n_{k}}{n} Dis_{k}$, where $n_{k}$ represents the number of $k$-th client's local samples and $n = \sum_{k=1}^{K}n_{k}$.
In AFL, the target distribution $\widehat{U}$ can be modeled as an unknown mixture of the distributions $\left \{ Dis_{1}, Dis_{2}, \ldots, Dis_{K} \right \}$.
That is, $\widehat{U} = \sum_{k=1}^{K} \widehat{p}_{k} Dis_{k}$, where $\widehat{p}_{k} \ge 0$ and $\sum_{k=1}^{K} \widehat{p}_{k}=1$.
AFL aims to optimize the model performance on $\widehat{U}$ for any possible choices of $\widehat{p}_{k}$ ($k = 1, \ldots, K$).

For a model trained by AFL, it can be used for many different agnostic target domains. 
Each agnostic target domain represents a distinct use case.
A good model in AFL is expected to have high AMP and high fairness across these multiple use cases in reality. 
Thus, a good AFL model should be able to achieve both high AMP and high fairness in heterogeneous AFL. 
In the following, we theoretically prove that a model trained by FedKF can directly have both high AMP and high fairness in heterogeneous AFL.

\newtheorem{lmm}{\textbf{Lemma}}
\begin{lmm}\label{lmm1}
We denote by $w$ a trained model via using FedKF.
In heterogeneous FL with FedKF, let $\Omega = \{Dis_{1}, \ldots, Dis_{K}\}$ represent a set of the client distributions.
$WLP_w^{\Omega}$ denotes the worst-case local performance of $w$ on $\Omega$.
Suppose that $w$ is used for an arbitrary agnostic domain $\widehat{U}$, let $MP_w^{\widehat{U}}$ be the model performance on the agnostic domain $\widehat{U}$.
It holds that 
\begin{equation}
\label{eq-43-1}
MP_w^{\widehat{U}}\geq WLP_w^{\Omega}.
\end{equation}
\end{lmm}

\begin{proof}
Let $a_{w}^{k}$ represent the test accuracy on distribution $Dis_{k}$ ($k = 1, \ldots, K$).
Then, $WLP_w^{\Omega}$ is given by
\begin{equation}
\label{eq-43-1.5}
WLP_w^{\Omega} = \min \{ a_{w}^{1}, \ldots, a_{w}^{K}\}.
\end{equation}
For $MP_w^{\widehat{U}}$, we have 
\begin{equation}
\label{eq-43-2}
MP_w^{\widehat{U}} = \sum_{k=1}^{K} \widehat{p}_{k} a_{w}^{k}.
\end{equation}
According to Eq. (\ref{eq-43-1.5}), $WLP_w^{\Omega}$ is the lower bound for  $a_{w}^{k}$ ($k = 1, \ldots, K$). 
Thus, substituting $WLP_w^{\Omega}$ for $a_{w}^{k}$ ($k = 1, \ldots, K$) in Eq. (\ref{eq-43-2}), it holds that
$$
MP_w^{\widehat{U}} \ge \sum_{k=1}^{K} \widehat{p}_{k} WLP_w^{\Omega} = WLP_w^{\Omega} \sum_{k=1}^{K} \widehat{p}_{k} = WLP_w^{\Omega}.
$$
\end{proof}

\newtheorem{thm}{\bf Theorem}[section]
\begin{thm}\label{thm1}
Suppose that there are multiple (e.g., $Q$) arbitrary agnostic target domains $\widehat{U}_i$ ($i=1,\ldots, Q$).
Let $WLP_w^{\widehat{\Omega}}$ be the worst-case performance on these agnostic target domains, where $\widehat{\Omega} = \{\widehat{U}_1, \ldots, \widehat{U}_Q\}$.
It holds that 
\begin{equation}
\label{eq-43-3}
WLP_w^{\widehat{\Omega}} \geq WLP_w^{\Omega},
\end{equation}
\end{thm}

\begin{proof}
According to Lemma \ref{lmm1}, it holds that 
$MP_w^{\widehat{U}_i} \geq WLP_w^{\Omega}$
for any $i \in [Q]$.
Since $WLP_w^{\widehat{\Omega}} = \min \{ MP_w^{\widehat{U}_i}, \ldots, MP_w^{\widehat{U}_M} \}$, we have
$
WLP_w^{\widehat{\Omega}} \geq WLP_w^{\Omega}.
$

\end{proof}

According to Theorem \ref{thm1}, the WLP of a model trained by FedKF in heterogeneous FL is the lower bound of the WLP when the model is used for heterogeneous AFL. 
Given the fact that the WLP metric tends to measure the joint performance of AMP and FM, FedKF directly turns out to be a good solution to achieve high AMP and high fairness simultaneously in heterogeneous AFL. 
Since AFL has more use cases, FedKF can also be applied to a broader range of use cases. 
Thus, FedKF has much broader impacts in reality.
\begin{figure*}[t]
\centering
\subfigure[$\alpha=1$]{
    \centering
    \includegraphics[width=0.31\linewidth]{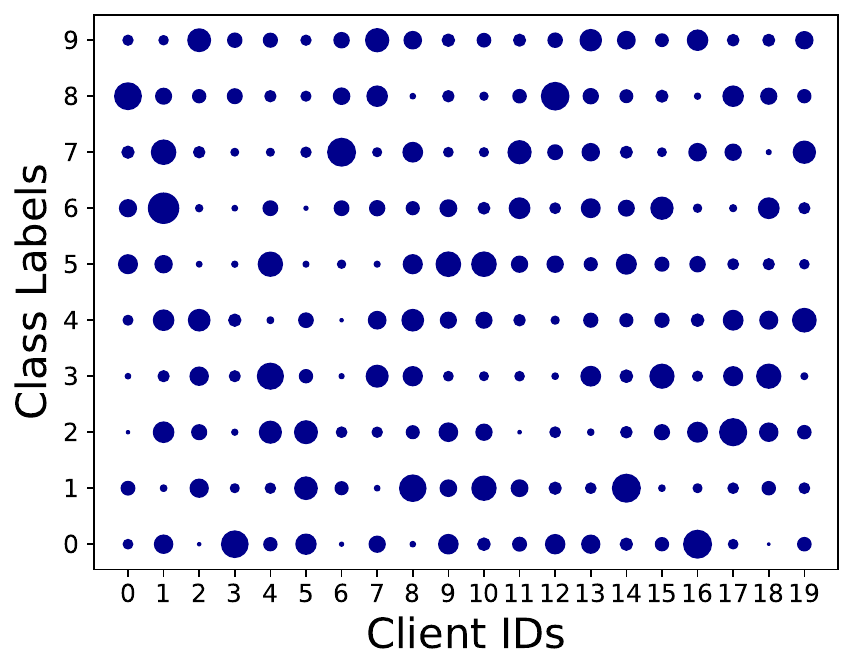}
}
\subfigure[$\alpha=0.1$]{
    \centering
    \includegraphics[width=0.31\linewidth]{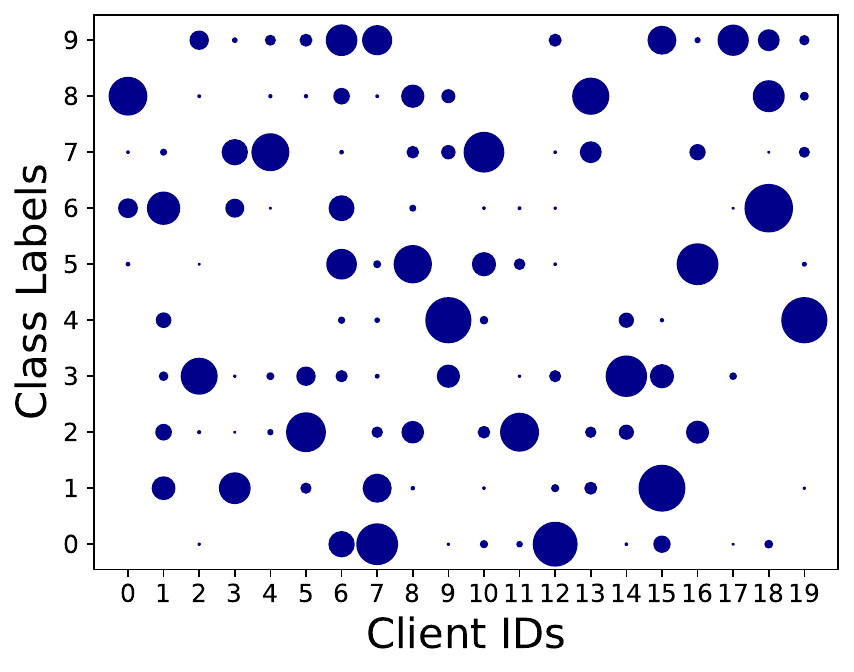}
}
\subfigure[$\alpha=0.01$]{
    \centering
    \includegraphics[width=0.31\linewidth]{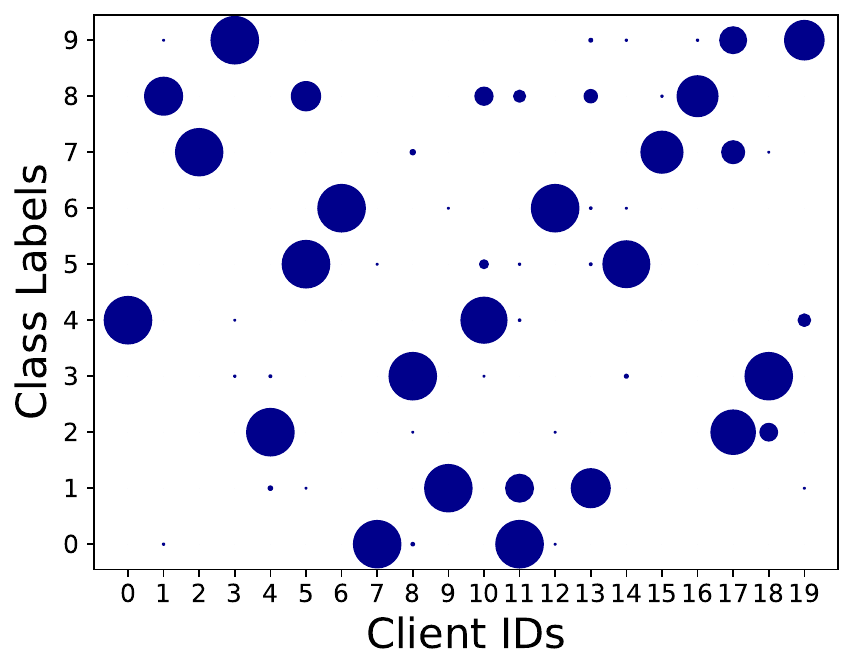}
}
\caption{Visualization of statistical heterogeneity among clients on CIFAR-10 dataset with different $\alpha$. The size of scattered points is proportional to the number of training samples for a label available on the client.}
\label{Fig::DistriCIFAR10}
\end{figure*}

\section{Experiments}
\label{experiments}

This section first introduces the experiment setup.
Then, the experimental results are reported.

\subsection{Experiment Setup}

\noindent
\textbf{FedKF Variants.}
When FedKF training is finished, either the ACA model or the
OCA model can serve as the final model to be used.
Depending on which model is used as the final model, FedKF has two variants: FedKF (ACA) and FedKF (OCA).
Besides, during the local training in FedKF, the OCA model serves as the teacher model, and the ACA model serves as the local model, so in each communication round, the amount of data in downlink communication (from server to client) is doubled compared with FedAvg.
If we use the ACA model to serve as both the teacher model and the local model, then the amount of data in downlink communication does not increase in each communication round.
This FedKF variant is denoted as FedKF-.
We call it the communication-efficient variant in this paper.
Note that in FedKF-, the student model is trained based on both the KL loss and the CE loss, so the student model (starting from the ACA model) is evolving along with the local training process. In contrast, the teacher model (always the ACA model) is fixed during the local training.
To sum up, we have four FedKF variants: FedKF (ACA), FedKF (OCA), FedKF- (ACA), and FedKF- (OCA).

\noindent
\textbf{Solutions in Comparison.}
We compare four FedKF variants with previous FL algorithms, including FedAvg \cite{mcmahan2017communication}, FedProx \cite{li2020federated}, FedGen \cite{zhu2021data}, FedGKD \cite{yao2021local}, and q-FFL \cite{li2019fair}.
For q-FFL, we use FedAvg as its optimization method, and it is also called q-FedAvg in \cite{li2019fair}.

\noindent
\textbf{Datasets.}
We conduct experiments on three datasets, including EMNIST \cite{cohen2017emnist}, CIFAR-10, and CIFAR-100 \cite{krizhevsky2009learning}.
For EMNIST, we only use a subset of the dataset by randomly sampling $10\%$ from each class.
Each client's local dataset is split into $80\%$ training set and $20\%$ testing set randomly.
Following previous works \cite{lin2020ensemble, zhu2021data, li2021model, yao2021local}, we use Dirichlet distribution to model heterogeneous data.
The Dirichlet distribution $\mathbf{Dir}_{K}(\alpha)$ has a adjustable concentration parameter $\alpha$.
A smaller $\alpha$ implies a higher data heterogeneity across different clients.
For example, the statistical heterogeneity among clients on CIFAR-10 with different concentration parameters $\alpha$ is shown in Fig. \ref{Fig::DistriCIFAR10}.

\noindent
\textbf{Implementation \& Training Details.}
The proposed FedKF and solutions in comparison are all implemented in PyTorch \cite{NEURIPS2019_9015} and evaluated on a Linux server with two TITAN RTX GPUs.
Since the learning process is exactly the same, the performance metrics measured are accurate in our experiments.

For the shared global model, two different neural network models are used.
ResNet-8 \cite{he2016deep} is used for CIFAR-10/100 and LeNet-5 \cite{lecun1989backpropagation} is used for EMNIST.
Note that Batch Normalization (BN) fails on heterogeneous training data due to the statistics of running mean and variance for the clients' data \cite{hsieh2020non}; we replace BN with Group Normalization (GN) to produce stabler results and set the number of channels of each group as 1.
For the generator used in data-free KD, we use a deep convolutional generator used in \cite{radford2015unsupervised} and replace $\mathrm{tanh}$ activation function in the last layer by $\mathrm{sigmoid}$.

The optimizers used in training local generator and local model are Adam and SGD, respectively.
For local generator training, the learning rate is set to 0.001.
For local model training, the learning rate is set to 0.01 and 0.1 for LeNet-5 and ResNet-8, respectively.
For FL learning, we run a total of 100 communication rounds.
We set the number of all clients $K$ to 20 and the selection rate $C$ to 20\%, which means there are 4 clients selected as active ones in each round.
The number of local update epochs $E$ is set to 10, and local batch size $B$ is set to 64.
For FedKF, we set $\gamma$, $\lambda_{1}$, and $\lambda_{2}$ to 1, 0.1, and 0.1, respectively, for all the datasets.

To compare with FedProx, we tune FedProx's parameter $\mu$ from \{0.00001, 0.0001, 0.001, 0.01, 0.1, 1\} and report the best result.
For FedProx, the best $\mu$ for CIFAR-10, CIFAR-100, and EMNIST are 0.001, 0.0001, and 0.001, respectively.
To compare with FedGKD, following \cite{yao2021local}, we set the default buffer size as 5.
For FedGKD, we tune FedGKD's parameter $\gamma$ from \{0.001, 0.01, 0.1, 0.2, 0.5, 1\}.
The best $\gamma$ for CIFAR-10, CIFAR-100, and EMNIST are 0.2, 0.2, and 0.001, respectively.
To compare with q-FFL, we tune q-FFL's parameter $q$ from \{0.00001, 0.0001, 0.001, 0.01, 0.1, 1\}.
The best $q$ for CIFAR-10, CIFAR-100, and EMNIST are 0.0001, 0.0001, and 0.0001, respectively.

\begin{table*}[t]
\centering
\caption{Performance of different solutions on EMNIST.}
\label{tab:EMNIST}
\resizebox{2\columnwidth}{!}{
\renewcommand{\arraystretch}{1.4}
\begin{tabular}{c|ccc|ccc|ccc}
\hline
\multirow{2}{*}{Solutions} & \multicolumn{3}{c|}{$\alpha = 1$} & \multicolumn{3}{c|}{$\alpha = 0.1$} & \multicolumn{3}{c}{$\alpha = 0.01$} \\
\cline{2-10}
                      & AMP (\%)                        & FM ($\times 10^{-3}$)              & WLP (\%)                         & AMP (\%)                         & FM ($\times 10^{-3}$)             & WLP (\%)                         & AMP (\%)                         & FM ($\times 10^{-2}$)             & WLP (\%) \\
\hline
\hline
FedAvg                & 83.78 $\!\pm\!$ 0.24             & 1.477 $\!\pm\!$ 0.132             & 76.46 $\!\pm\!$ 1.66             & 75.61 $\!\pm\!$ 0.92             & 7.972 $\!\pm\!$ 0.703             & 57.07 $\!\pm\!$ 1.03             & 55.42 $\!\pm\!$ 2.78             & 4.871 $\!\pm\!$ 0.976             & 5.61 $\!\pm\!$  2.86 \\
\hline
FedProx               & 83.87 $\!\pm\!$ 0.18             & 2.008 $\!\pm\!$ 0.167             & 75.24 $\!\pm\!$ 1.32             & 75.59 $\!\pm\!$ 0.99             & 7.844 $\!\pm\!$ 0.610             & 56.75 $\!\pm\!$ 4.09             & 55.70 $\!\pm\!$ 1.82             & 6.535 $\!\pm\!$ 1.324             & 3.96 $\!\pm\!$ 1.96 \\
\hline
FedGen                & 84.81 $\!\pm\!$ 0.36             & 1.379 $\!\pm\!$ 0.206             & 78.25 $\!\pm\!$ 2.29             & 77.37 $\!\pm\!$ 1.62             & 5.405 $\!\pm\!$ 1.559             & 61.20 $\!\pm\!$ 6.59             & 56.07 $\!\pm\!$ 3.94             & 5.951 $\!\pm\!$ 1.194             & 9.08 $\!\pm\!$ 3.92 \\
\hline
FedGKD                & 83.63 $\!\pm\!$ 0.28             & 1.463 $\!\pm\!$ 0.124             & 74.78 $\!\pm\!$ 1.34             & 75.93 $\!\pm\!$ 0.67             & 6.477 $\!\pm\!$ 0.504             & 56.02 $\!\pm\!$ 2.44             & 57.86 $\!\pm\!$ 1.48             & 4.266 $\!\pm\!$ 0.846             & 7.92 $\!\pm\!$ 2.32 \\
\hline
q-FFL                 & 84.07 $\!\pm\!$ 0.34             & 1.729 $\!\pm\!$ 0.182             & 76.58 $\!\pm\!$ 1.78             & 75.67 $\!\pm\!$ 0.02             & 6.740 $\!\pm\!$ 1.524             & 57.25 $\!\pm\!$ 2.30             & 54.90 $\!\pm\!$ 2.36             & 3.128 $\!\pm\!$ 0.784             & 18.70 $\!\pm\!$ 4.36 \\
\hline
\textbf{FedKF-} (ACA) & 85.18 $\!\pm\!$ 0.38             & \textbf{1.064 $\!\pm\!$ 0.109}    & \textbf{78.38 $\!\pm\!$ 0.90}    & 82.82 $\!\pm\!$ 0.50             & \underline{2.753 $\!\pm\!$ 0.503} & \underline{72.71 $\!\pm\!$ 1.09} & \underline{74.44 $\!\pm\!$ 0.95} & \underline{1.667 $\!\pm\!$ 0.137} & \underline{39.92 $\!\pm\!$ 2.27} \\
\hline
\textbf{FedKF-} (OCA) & 85.26 $\!\pm\!$ 0.42             & \underline{1.211 $\!\pm\!$ 0.217} & \underline{78.35 $\!\pm\!$ 0.95} & \textbf{83.36 $\!\pm\!$ 0.62}    & \textbf{2.358 $\!\pm\!$ 0.287}    & \textbf{73.76 $\!\pm\!$ 0.80}    & \textbf{76.20 $\!\pm\!$ 0.78}    & \textbf{1.283 $\!\pm\!$ 0.157}    & \textbf{41.91 $\!\pm\!$ 4.68} \\
\hline
\textbf{FedKF } (ACA) & \textbf{85.54 $\!\pm\!$ 0.37}    & 1.717 $\!\pm\!$ 0.251             & 77.39 $\!\pm\!$ 1.02             & 82.76 $\!\pm\!$ 0.16             & 3.049 $\!\pm\!$ 0.848             & 72.41 $\!\pm\!$ 0.12             & 72.33 $\!\pm\!$ 1.12             & 2.153 $\!\pm\!$ 0.478             & 38.12 $\!\pm\!$ 3.42 \\
\hline
\textbf{FedKF } (OCA) & \underline{85.46 $\!\pm\!$ 0.29} & 1.540 $\!\pm\!$ 0.192             & 78.26 $\!\pm\!$ 0.88             & \underline{82.94 $\!\pm\!$ 0.12} & 2.978 $\!\pm\!$ 0.383             & 71.53 $\!\pm\!$ 1.57             & 72.73 $\!\pm\!$ 1.09             & 2.780 $\!\pm\!$ 0.574             & 36.73 $\!\pm\!$ 3.86 \\
\hline
\end{tabular}}
\end{table*}

\begin{table*}[t]
\centering
\caption{Performance of different solutions on CIFAR-10.}
\label{tab:CIFAR10}
\resizebox{2\columnwidth}{!}{
\renewcommand{\arraystretch}{1.4}
\begin{tabular}{c|ccc|ccc|ccc}
\hline
\multirow{2}{*}{Solutions} & \multicolumn{3}{c|}{$\alpha = 1$} & \multicolumn{3}{c|}{$\alpha = 0.1$} & \multicolumn{3}{c}{$\alpha = 0.01$} \\
\cline{2-10}
                       & AMP (\%)                         & FM ($\times 10^{-3}$)             & WLP (\%)                         & AMP (\%)                         & FM ($\times 10^{-2}$)             & WLP (\%)                         & AMP (\%)                         & FM ($\times 10^{-2}$)             & WLP (\%) \\
\hline
\hline
FedAvg                 & 74.28 $\!\pm\!$ 0.28             & 1.014 $\!\pm\!$ 0.148             & 68.23 $\!\pm\!$ 0.43             & 61.89 $\!\pm\!$ 0.81             & 1.807 $\!\pm\!$ 0.240             & 34.41 $\!\pm\!$ 7.74             & 38.48 $\!\pm\!$ 0.43             & 7.914 $\!\pm\!$ 0.612             & 1.82 $\!\pm\!$ 1.21 \\
\hline
FedProx                & 74.00 $\!\pm\!$ 0.16             & 1.079 $\!\pm\!$ 0.126             & 67.87 $\!\pm\!$ 0.23             & 62.25 $\!\pm\!$ 0.77             & 1.764 $\!\pm\!$ 0.365             & 36.09 $\!\pm\!$ 10.66            & 38.30 $\!\pm\!$ 0.64             & 7.979 $\!\pm\!$ 0.742             & 1.24 $\!\pm\!$ 0.63 \\
\hline
FedGen                 & 74.52 $\!\pm\!$ 0.22             & 0.778 $\!\pm\!$ 0.110             & 68.56 $\!\pm\!$ 0.06             & 62.92 $\!\pm\!$ 0.75             & 1.916 $\!\pm\!$ 0.342             & 34.00 $\!\pm\!$ 6.26             & 40.34 $\!\pm\!$ 0.51             & 6.673 $\!\pm\!$ 0.689             & 2.63 $\!\pm\!$ 2.41 \\
\hline
FedGKD                 & 74.83 $\!\pm\!$ 0.18             & 0.786 $\!\pm\!$ 0.096             & 69.40 $\!\pm\!$ 0.38             & 63.98 $\!\pm\!$ 0.38             & 1.508 $\!\pm\!$ 0.682             & 40.04 $\!\pm\!$ 5.72             & 39.14 $\!\pm\!$ 0.72             & 5.478 $\!\pm\!$ 0.574             & 2.40 $\!\pm\!$ 1.76 \\
\hline
q-FFL                  & 73.96 $\!\pm\!$ 0.31             & \textbf{0.716 $\!\pm\!$ 0.102}    & 68.59 $\!\pm\!$ 0.19             & 61.49 $\!\pm\!$ 1.26             & 1.853 $\!\pm\!$ 0.451             & 33.02 $\!\pm\!$ 8.20             & 37.75 $\!\pm\!$ 0.58             & 3.655 $\!\pm\!$ 0.484             & 1.26 $\!\pm\!$ 0.78 \\
\hline
\textbf{FedKF -} (ACA) & 75.19 $\!\pm\!$ 0.07             & \underline{0.737 $\!\pm\!$ 0.158} & \textbf{70.14 $\!\pm\!$ 0.49}    & 67.97 $\!\pm\!$ 0.79             & 1.333 $\!\pm\!$ 0.075             & 47.72 $\!\pm\!$ 3.30             & 47.98 $\!\pm\!$ 0.97             & 5.576 $\!\pm\!$ 0.347             & 3.96 $\!\pm\!$ 3.67 \\
\hline
\textbf{FedKF -} (OCA) & \textbf{75.62 $\!\pm\!$ 0.13}    & 1.150 $\!\pm\!$ 0.138             & 69.13 $\!\pm\!$ 0.59             & \underline{69.88 $\!\pm\!$ 0.29} & \underline{0.874 $\!\pm\!$ 0.187} & \underline{53.89 $\!\pm\!$ 2.54} & \underline{54.41 $\!\pm\!$ 0.54} & \underline{3.063 $\!\pm\!$ 0.191} & \underline{26.09 $\!\pm\!$ 4.84} \\
\hline
\textbf{FedKF} (ACA)   & 75.23 $\!\pm\!$ 0.17             & 0.971 $\!\pm\!$ 0.117             & \underline{69.68 $\!\pm\!$ 0.42} & 67.86 $\!\pm\!$ 0.39             & 1.271 $\!\pm\!$ 0.146             & 45.12 $\!\pm\!$ 2.75             & 47.89 $\!\pm\!$ 0.49             & 5.817 $\!\pm\!$ 0.536             & 5.68 $\!\pm\!$ 1.28 \\
\hline
\textbf{FedKF} (OCA)   & \underline{75.54 $\!\pm\!$ 0.09} & 1.197 $\!\pm\!$ 0.078             & 69.03 $\!\pm\!$ 0.27             & \textbf{70.11 $\!\pm\!$ 0.74}    & \textbf{0.835 $\!\pm\!$ 0.110}    & \textbf{55.18 $\!\pm\!$ 1.32}    & \textbf{54.74 $\!\pm\!$ 0.66}    & \textbf{2.792 $\!\pm\!$ 0.263}    & \textbf{29.86 $\!\pm\!$ 5.51} \\
\hline
\end{tabular}}
\end{table*}

\begin{table*}[t]
\centering
\caption{Performance of different solutions on CIFAR-100.}
\label{tab:CIFAR100}
\resizebox{2\columnwidth}{!}{
\renewcommand{\arraystretch}{1.4}
\begin{tabular}{c|ccc|ccc|ccc}
\hline
\multirow{2}{*}{Solutions} & \multicolumn{3}{c|}{$\alpha = 1$} & \multicolumn{3}{c|}{$\alpha = 0.1$} & \multicolumn{3}{c}{$\alpha = 0.01$} \\
\cline{2-10}
                       & AMP (\%)                         & FM ($\times 10^{-3}$)             & WLP (\%)                         & AMP (\%)                         & FM ($\times 10^{-3}$)             & WLP (\%)                         & AMP (\%)                         & FM ($\times 10^{-2}$)             & WLP (\%) \\
\hline
\hline
FedAvg                 & 36.92 $\!\pm\!$ 0.42             & 1.549 $\!\pm\!$ 0.274             & 29.53 $\!\pm\!$ 0.62             & 29.37 $\!\pm\!$ 0.58             & 6.265 $\!\pm\!$ 1.206             & 17.70 $\!\pm\!$ 0.64             & 17.41 $\!\pm\!$ 0.12             & 1.660 $\!\pm\!$ 0.152             & 3.37 $\!\pm\!$ 0.33 \\
\hline
FedProx                & 36.63 $\!\pm\!$ 0.49             & 0.958 $\!\pm\!$ 0.241             & 30.60 $\!\pm\!$ 0.74             & 29.62 $\!\pm\!$ 0.51             & 6.131 $\!\pm\!$ 0.455             & 17.11 $\!\pm\!$ 1.11             & 17.75 $\!\pm\!$ 0.19             & 1.828 $\!\pm\!$ 0.134             & 3.41 $\!\pm\!$ 0.43 \\
\hline
FedGen                 & 40.23 $\!\pm\!$ 0.38             & 1.154 $\!\pm\!$ 0.222             & 32.21 $\!\pm\!$ 0.80             & 32.08 $\!\pm\!$ 0.51             & 5.378 $\!\pm\!$ 0.938             & 19.37 $\!\pm\!$ 4.49             & 18.06 $\!\pm\!$ 0.05             & 1.727 $\!\pm\!$ 0.122             & 1.93 $\!\pm\!$ 0.39 \\
\hline
FedGKD                 & 38.45 $\!\pm\!$ 0.56             & 1.302 $\!\pm\!$ 0.213             & 32.97 $\!\pm\!$ 0.58             & 32.23 $\!\pm\!$ 0.46             & 3.960 $\!\pm\!$ 1.218             & 22.51 $\!\pm\!$ 1.02             & 15.77 $\!\pm\!$ 0.18             & 1.626 $\!\pm\!$ 0.146             & 1.06 $\!\pm\!$ 0.34 \\
\hline
q-FFL                  & 36.12 $\!\pm\!$ 0.31             & \textbf{0.729 $\!\pm\!$ 0.178}    & 31.26 $\!\pm\!$ 0.67             & 29.43 $\!\pm\!$ 0.96             & 5.781 $\!\pm\!$ 2.946             & 17.59 $\!\pm\!$ 2.74             & 16.82 $\!\pm\!$ 0.22             & 1.757 $\!\pm\!$ 0.118             & 1.06 $\!\pm\!$ 0.19 \\
\hline
\textbf{FedKF -} (ACA) & 39.86 $\!\pm\!$ 0.49             & 0.952 $\!\pm\!$ 0.113             & 33.79 $\!\pm\!$ 0.58             & 33.73 $\!\pm\!$ 0.73             & 4.420 $\!\pm\!$ 1.189             & 24.50 $\!\pm\!$ 1.21             & 20.30 $\!\pm\!$ 0.58             & 1.404 $\!\pm\!$ 0.134             & 3.34 $\!\pm\!$ 0.77 \\
\hline
\textbf{FedKF -} (OCA) & \underline{41.13 $\!\pm\!$ 0.41} & 1.019 $\!\pm\!$ 0.329             & \underline{35.51 $\!\pm\!$ 1.08} & \underline{37.43 $\!\pm\!$ 0.65} & \underline{1.869 $\!\pm\!$ 0.192} & \underline{28.30 $\!\pm\!$ 1.21} & \underline{23.45 $\!\pm\!$ 0.66} & \underline{0.999 $\!\pm\!$ 0.198} & \underline{8.03 $\!\pm\!$ 0.69} \\
\hline
\textbf{FedKF} (ACA)   & 40.53 $\!\pm\!$ 0.37             & \underline{0.906 $\!\pm\!$ 0.152} & 34.16 $\!\pm\!$ 0.72             & 34.11 $\!\pm\!$ 0.63             & 4.273 $\!\pm\!$ 0.709             & 22.30 $\!\pm\!$ 1.31             & 21.32 $\!\pm\!$ 0.78             & 1.900 $\!\pm\!$ 0.055             & 4.50 $\!\pm\!$ 0.16 \\
\hline
\textbf{FedKF} (OCA)   & \textbf{41.30 $\!\pm\!$ 0.43}    & 0.978 $\!\pm\!$ 0.137             & \textbf{36.43 $\!\pm\!$ 0.88}    & \textbf{38.08 $\!\pm\!$ 0.19}    & \textbf{1.775 $\!\pm\!$ 0.109}    & \textbf{29.40 $\!\pm\!$ 0.70}    & \textbf{24.06 $\!\pm\!$ 0.50}    & \textbf{0.876 $\!\pm\!$ 0.089}    & \textbf{8.71 $\!\pm\!$ 0.97} \\
\hline
\end{tabular}}
\end{table*}

\begin{figure*}[t]
\centering
\subfigure[EMNIST]{
    \centering
    \includegraphics[width=0.31\linewidth]{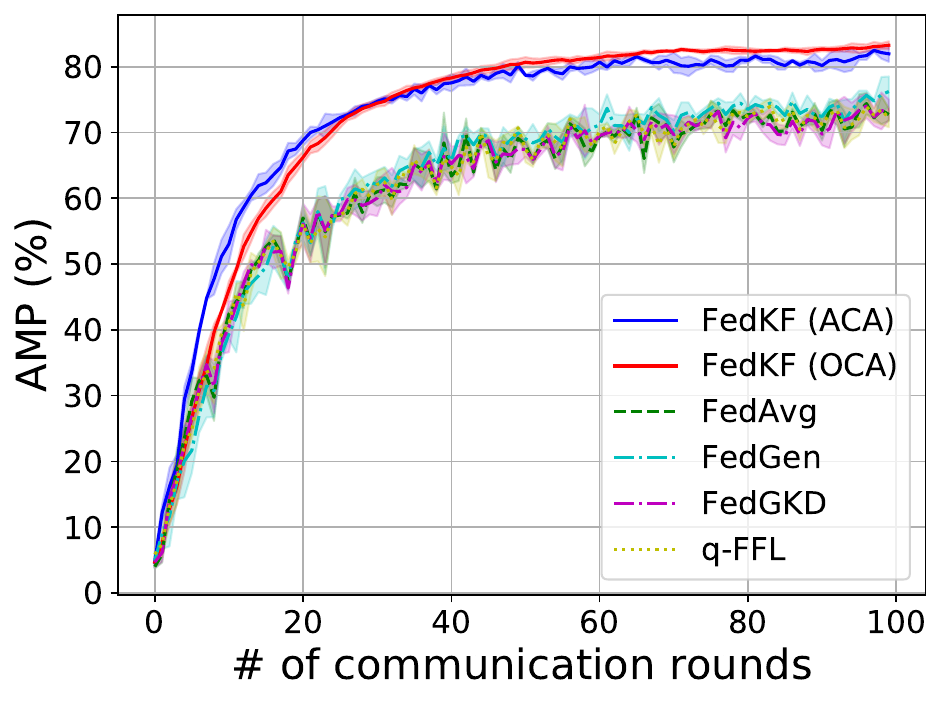}
}
\subfigure[CIFAR-10]{
    \centering
    \includegraphics[width=0.31\linewidth]{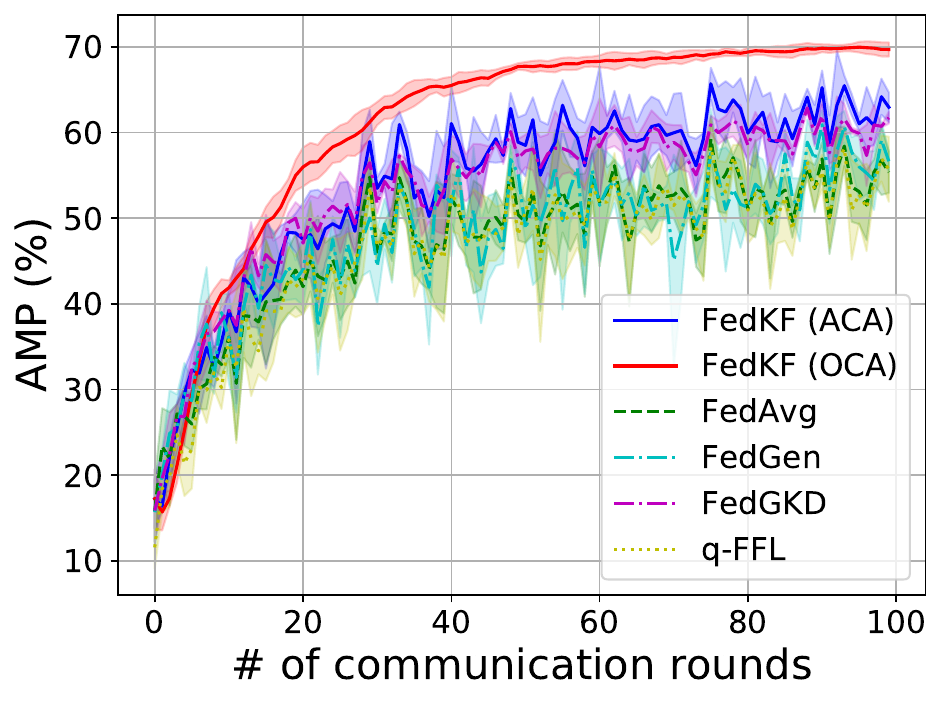}
}
\subfigure[CIFAR-100]{
    \centering
    \includegraphics[width=0.31\linewidth]{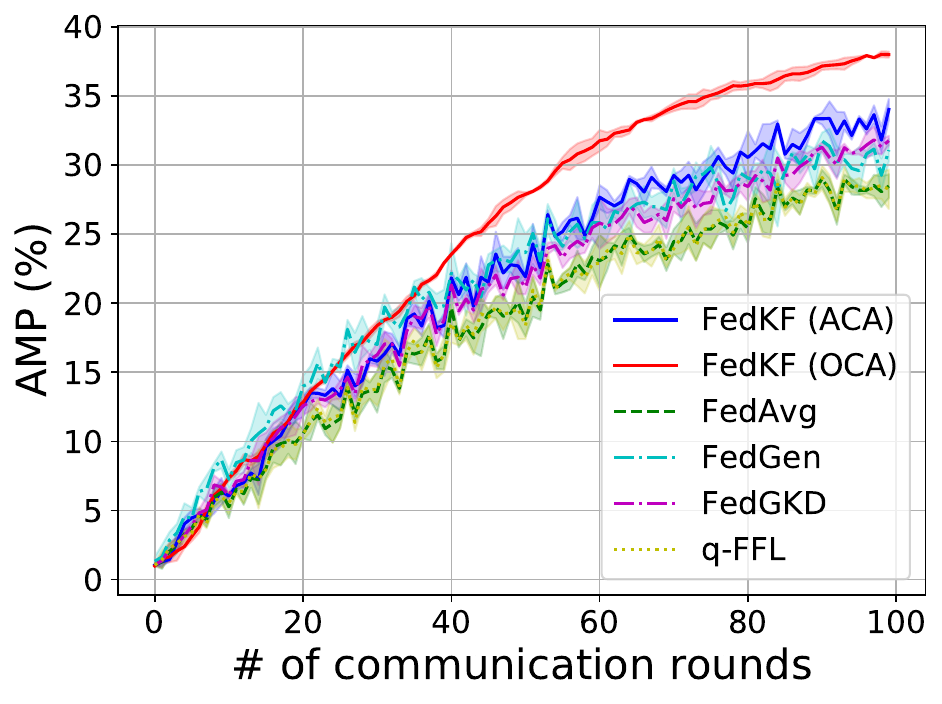}
}

\subfigure[EMNIST]{
    \centering
    \includegraphics[width=0.31\linewidth]{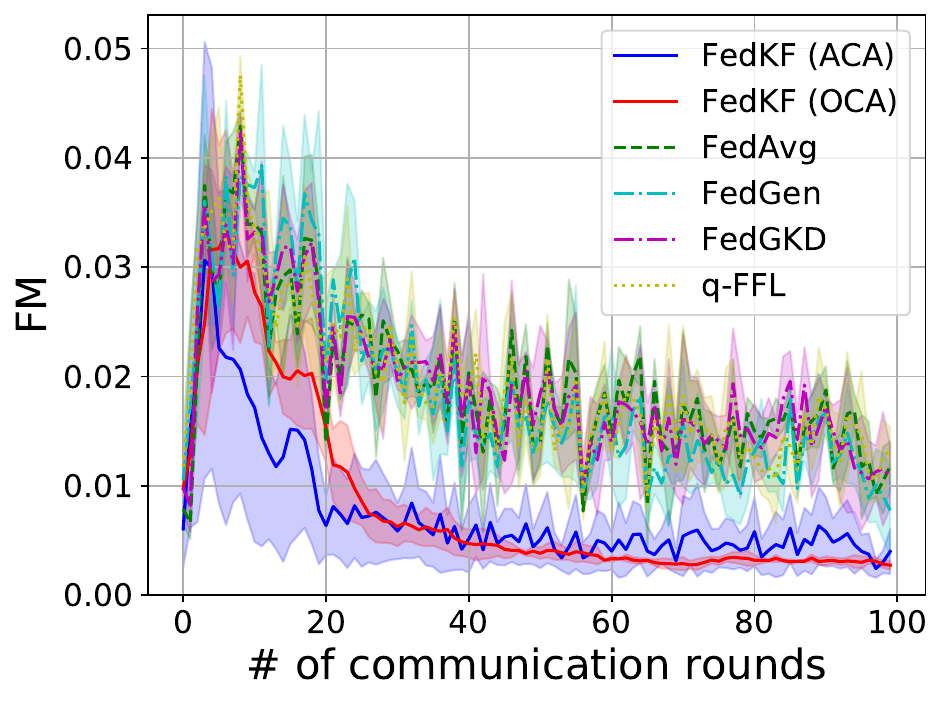}
}
\subfigure[CIFAR-10]{
    \centering
    \includegraphics[width=0.31\linewidth]{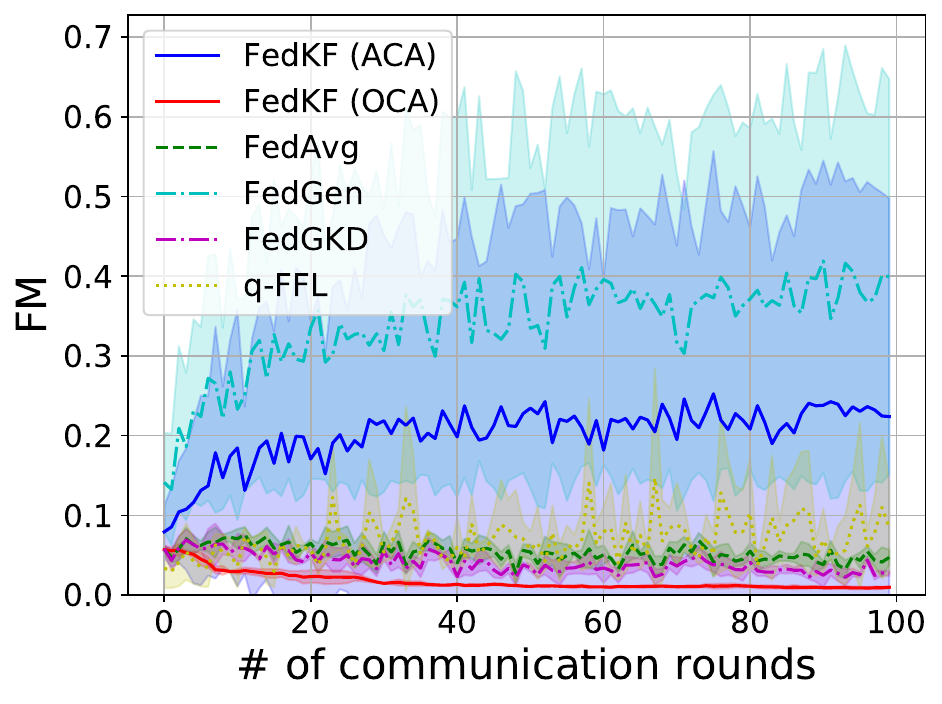}
}
\subfigure[CIFAR-100]{
    \centering
    \includegraphics[width=0.31\linewidth]{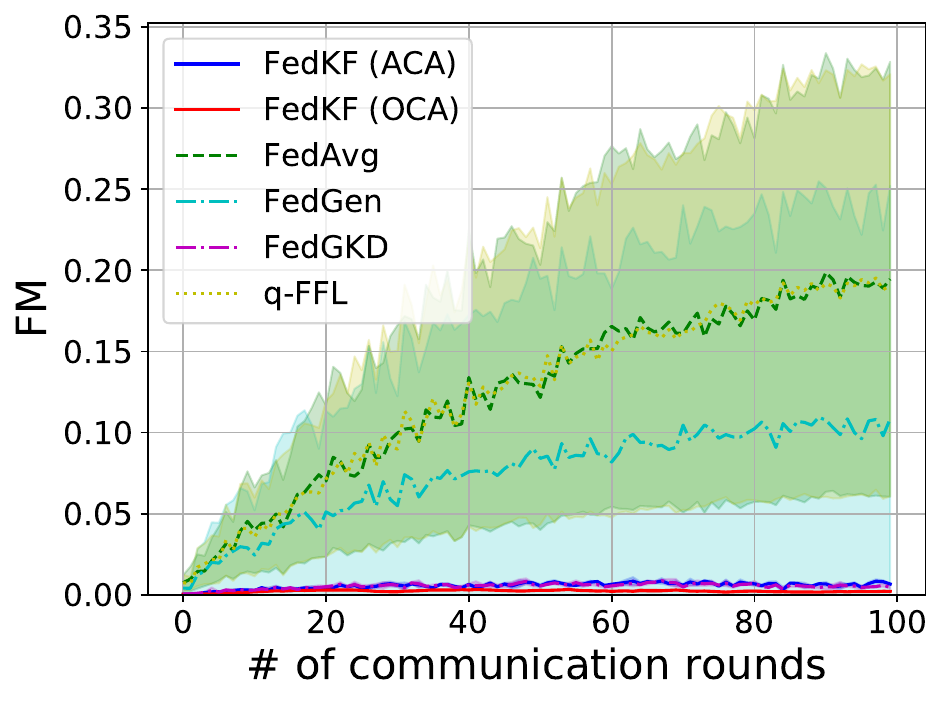}
}
\subfigure[EMNIST]{
    \centering
    \includegraphics[width=0.31\linewidth]{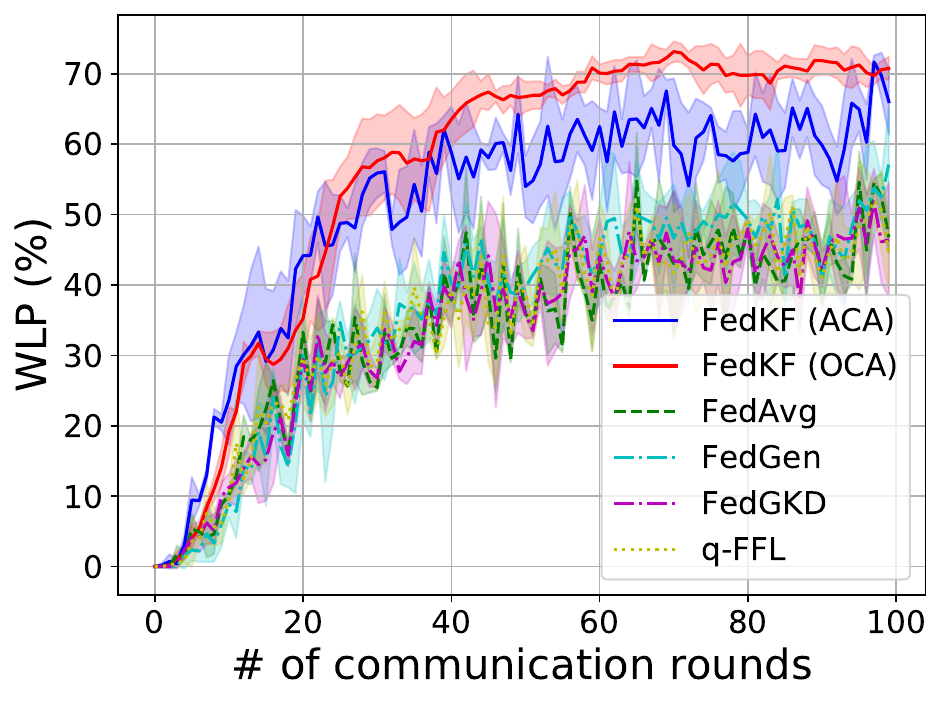}
}
\subfigure[CIFAR-10]{
    \centering
    \includegraphics[width=0.31\linewidth]{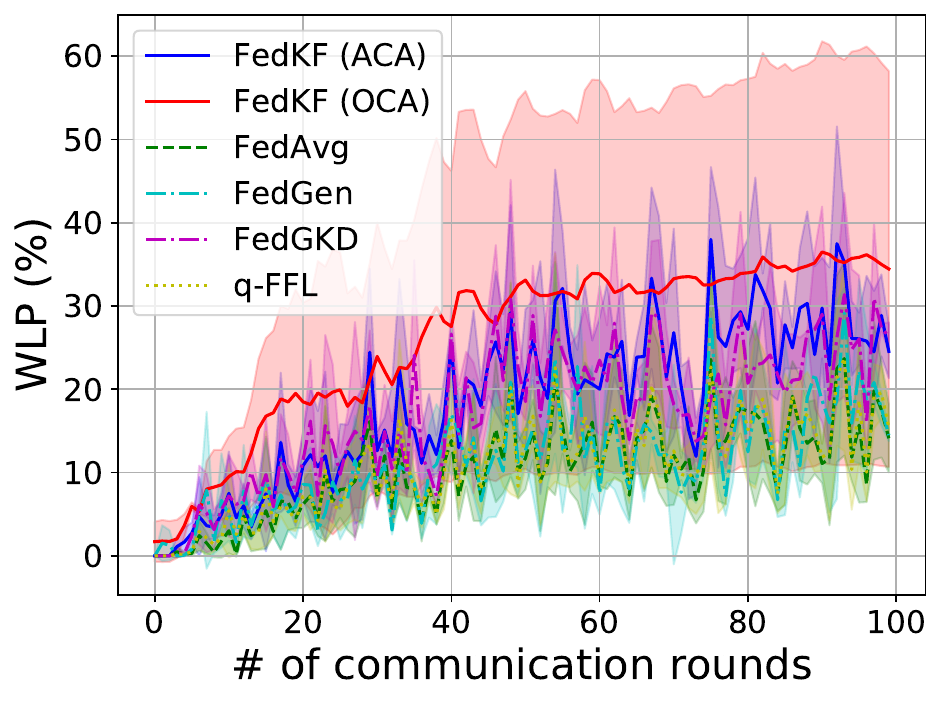}
}
\subfigure[CIFAR-100]{
    \centering
    \includegraphics[width=0.31\linewidth]{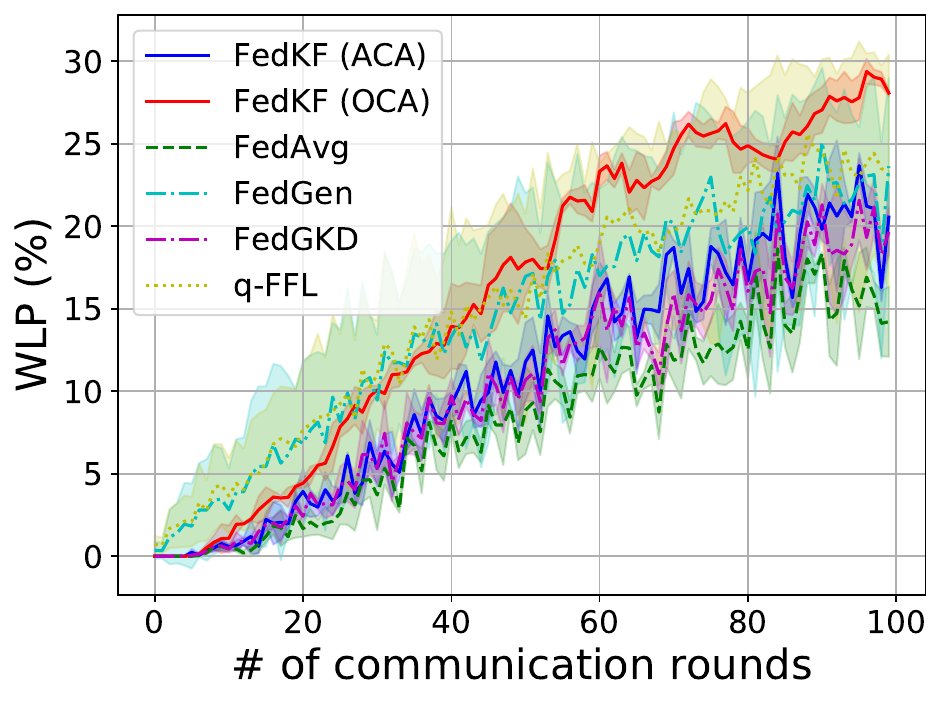}
}
\caption{Performance v.s. number of communication rounds on different datasets with $\alpha = 0.1$.}
\label{Fig::NumberofC}
\end{figure*}

\subsection{Experimental Results}
\label{sect:results}

\noindent
\textbf{Performance Metrics.}
Table \ref{tab:EMNIST}-\ref{tab:CIFAR100} show the performance metrics (i.e., AMP, FM, and WLP) of different solutions for three degrees of data heterogeneity ($\alpha=1, 0.1, 0.01$).
In the tables, for each case, the best result is marked as \textbf{bold} font, and the second best is marked using \underline{underline}.
We have the following six findings.
\begin{itemize}[leftmargin=0.4cm]
\item \textbf{F1:} FedKF-(OCA) performance is the best when $\alpha=0.1$ and 0.01 on EMNIST, while FedKF (OCA) performance is the best when $\alpha=0.1$ and 0.01 on both CIFRA-10 and CIFAR-100 datasets.

\item \textbf{F2:} When $\alpha=1$ on three datasets,
at least one of the four FedKF variants can rank in the top 2 (in terms of AMP, FM, and WLP) among all solutions in comparison.

\item \textbf{F3:} FedKF has better performance than FedGen on all three datasets.
FedGen allows the leakage of local label counts to the server, meaning that the server knows exactly how each local dataset is heterogeneous. Thus, it violates the privacy-preserving property (see Definition \ref{def:privacy}).
FedKF outperforms FedGen while still guaranteeing the privacy-preserving property.

\item \textbf{F4:} Under most cases, the OCA variants have better performance than their corresponding ACA variants.
The superiority of the OCA variants (over the ACA variants) is more obvious with the decrease of $\alpha$ (i.e., higher data heterogeneity).

\item \textbf{F5:} Under most cases, the original FedKF has a better performance than its communication-efficient variants (i.e., FedKF-(ACA) and FedKF-(OCA)).

\item \textbf{F6:} Under most cases,  the superiority of FedKF variants (over other solutions) is more obvious with a decreasing of $\alpha$.
For example, on CIFAR-10, FedKF (OCA)'s WLP is (69.03\%-68.23\%)=0.8\% better than FedAvg when $\alpha=1$, whereas it is (29.86\%-1.82\%)=28.04\% better than FedAvg when $\alpha=0.01$.
Hence, FedKF is especially good at dealing with highly heterogeneous data.

\end{itemize}

\begin{table}[h]
\Huge
\centering
\caption{The number of rounds needed in different solutions to achieve the same AMP as running FedAvg for 100 rounds on different datasets with $\alpha = 0.1$. The speedup of a solution is computed against FedAvg.}
\label{tab:speedup}
\resizebox{1\columnwidth}{!}
{
\renewcommand{\arraystretch}{1.4}
\begin{tabular}{c|cc|cc|cc}
\hline
\multirow{2}{*}{Solutions} & \multicolumn{2}{c|}{EMNIST} & \multicolumn{2}{c|}{CIFAR-10} & \multicolumn{2}{c}{CIFAR-100} \\
\cline{2-7}
                     & \# of rounds & Speedup             & \# of rounds & Speedup             & \# of rounds & Speedup \\
\hline
\hline
FedAvg               & 100          & 1$\times$            & 100          & 1$\times$            & 100          & 1$\times$ \\
\hline
FedProx              & $>$100       & $<$1$\times$         & 98           & 1$\times$            & 96           & 1$\times$ \\
\hline
FedGen               & 77           & 1.3$\times$          & 91           & 1.1$\times$          & 74           & 1.4$\times$ \\
\hline
FedGKD               & $>$100       & $<$1$\times$         & 49           & 2.0$\times$          & 82           & 1.2$\times$ \\
\hline
q-FFL                & 100          & 1$\times$            & $>$100       & $<$1$\times$         & 100          & 1$\times$ \\
\hline
\textbf{FedKF-} (ACA) & \textbf{31} & \textbf{3.2}$\times$ & 35           & 2.9$\times$          & 68           & 1.5$\times$ \\
\hline
\textbf{FedKF-} (OCA) & \textbf{31} & \textbf{3.2}$\times$ & 30           & 3.3$\times$          & 55           & 1.8$\times$ \\
\hline
\textbf{FedKF} (ACA) & \textbf{31}  & \textbf{3.2}$\times$ & 34           & 2.9$\times$          & 66           & 1.5$\times$ \\
\hline
\textbf{FedKF} (OCA) & 32           & 3.1$\times$          & \textbf{28}  & \textbf{3.6}$\times$ & \textbf{52}  & \textbf{1.9}$\times$ \\
\hline
\end{tabular}
}
\end{table}

\noindent
\textbf{AMP v.s. Number of Rounds.}
As shown in Fig. \ref{Fig::NumberofC}, FedKF has a faster learning speed (i.e., has higher communication efficiency) than prior solutions when data is heterogeneous (i.e., $\alpha = 0.1$).
Specifically, the number of communication rounds to achieve the same AMP as running FedAvg for 100 rounds on different datasets with $\alpha = 0.1$ is shown in Table \ref{tab:speedup}.
We can observe that the number of communication rounds is significantly reduced by using FedKF.
For example, FedKF (OCA) needs 32, 28, and 52 communication rounds to achieve the same AMP as running FedAvg for 100 rounds on EMNIST, CIRFA-10, and CIFAR-100, respectively.
Thus, FedKF is much more communication-efficient than prior solutions.

\noindent
\textbf{FM v.s. Number of Rounds.}
The FM v.s. number of communication rounds on different datasets with $\alpha = 0.1$ is shown in Fig. \ref{Fig::NumberofC}.
We find that the FMs of FedKF (OCA) are always the lowest among all solutions on all datasets when training is finished, which means FedKF (OCA) achieves better fairness compared with previous solutions.
Besides, the FM curves of FedKF (OCA) are smoother and have smaller fluctuation amplitude compared with other solutions on all the datasets.
Furthermore, compared with FedKF (ACA), FedKF (OCA) has lower FMs, smoother FM curves, and smaller fluctuation amplitude, which means the technique T1 indeed improves model fairness in FedKF.

\noindent
\textbf{WLP v.s. Number of Rounds.}
The WLP v.s. number of communication rounds on different datasets with $\alpha = 0.1$ is shown in Fig. \ref{Fig::NumberofC}.
When training is finished, all the WLPs of our solution FedKF (OCA) are highest among all solutions on all datasets.
On EMNIST and CIFAR-10, the WLPs of both FedKF (ACA) and FedKF (OCA) are higher than previous solutions, while FedKF (OCA) has more smooth WLP curves compared with FedKF (ACA).

\noindent
\textbf{Communication Data Amount.}
For both FedKF (ACA) and FedKF (OCA), in each communication round, their uplink communication (from client to server) data amount is the same as FedAvg, while their downlink communication data amount is doubled compared with FedAvg.
For both FedKF- (ACA) and FedKF- (OCA), their communication data amount is the same as FedAvg.
Therefore, compared with FedKF, FedKF- slightly sacrifices performance to ensure there is no communication traffic increase in each communication round.
Besides, as shown in Table \ref{tab:speedup}, FedKF and FedKF- require much fewer communication rounds in learning compared with FedAvg.
Hence, in most cases, FedKF and FedKF- require less communication data amount than prior solutions.
Note that FedKF- has the least communication consumption since it requires half of the downlink communication data amount in each communication round compared with FedKF.

\noindent
\textbf{Client's Computation Overhead.}
We have a theoretical analysis of the client's computation overhead as follows.
In the local training phase of federated learning, the main computation is forward propagation (FP) and backpropagation (BP) through the models.
To compare the computation overhead of different solutions in the local training phase, we compare the numbers of the FP and BP operations through the models for each batch of the local data.
Table \ref{tab:ptimes} demonstrates the numbers of the FP and BP operations through the models for each batch of the local data in the local training phase in different solutions.
In FedGen, only the last layers of the local model (i.e., predictor) need 2 operations of both FP and BP for each batch of the local data, and the number of the FP and BP operations in the remaining layers is the same as FedAvg.
Therefore, we consider the numbers of the FP and BP operations through the classifier in FedGen as 1.5 and 1.5, respectively.
In FedGKD, the teacher model (the classifier) needs to output the logit of each sample for knowledge distillation, so the number of the FP operations through the classifier is 2.
In FedKF, the numbers of the FP and BP operations through the teacher model (the classifier) are 1 and 1, respectively, and those through the student model (the classifier) are 2 and 2.

\begin{table}[h]
\centering
\caption{The numbers of the FP and BP operations through the models for each batch of the local data in the local training phase in different solutions.}
\label{tab:ptimes}
{
\renewcommand{\arraystretch}{1.4}
\begin{tabular}{c|cc|cc}
\hline
\multirow{2}{*}{Solutions} & \multicolumn{2}{c|}{Classifier} & \multicolumn{2}{c}{Generator} \\
\cline{2-5}
                           & Forward        & Backward       & Forward     & Backward \\
\hline
\hline
FedAvg                     & 1              & 1              & 0           & 0 \\
\hline
FedProx                    & 1              & 1              & 0           & 0 \\
\hline
FedGen                     & 1.5            & 1.5            & 1           & 0 \\
\hline
FedGKD                     & 2              & 1              & 0           & 0 \\
\hline
q-FFL                      & 1              & 1              & 0           & 0 \\
\hline
\textbf{FedKF-}            & 3              & 3              & 1           & 1 \\
\hline
\textbf{FedKF}             & 3              & 3              & 1           & 1 \\
\hline
\end{tabular}
}
\end{table}

\noindent
\textbf{Time Consumption.}
To intuitively compare the time complexity of different solutions, we conduct experiments to report the time consumption of different solutions for 100 communication rounds on different datasets with $\alpha=0.1$.
Note that the active clients in the experiments serially execute \textbf{ClientUpdate} one by one, which means the time consumption in real-world applications is much less than that in the experiments due to its parallel execution of \textbf{ClientUpdate}.
As shown in Table \ref{tab:time}, while FedKF- and FedKF require every client independently to train an extra local generator and conduct local knowledge distillation, the time consumption of our solutions is about 1.8$\times$-1.9$\times$ longer than FedAvg on all three different datasets, which is acceptable.
Our solutions take more time to achieve better capabilities to handle data heterogeneity in FL.

\begin{table}[h]
\Huge
\centering
\caption{Time consumption of different solutions for 100 communication rounds on different datasets with $\alpha=0.1$. The speed-down computes its time consumption ratio compared to FedAvg.}
\label{tab:time}
\resizebox{1\columnwidth}{!}
{
\renewcommand{\arraystretch}{1.4}
\begin{tabular}{c|cc|cc|cc}
\hline
\multirow{2}{*}{Solutions} & \multicolumn{2}{c|}{EMNIST} & \multicolumn{2}{c|}{CIFAR-10} & \multicolumn{2}{c}{CIFAR-100} \\
\cline{2-7}
                & Time     & Speed-down  & Time      & Speed-down  & Time      & Speed-down \\
\hline
\hline
FedAvg          & 6min 43s & 1$\times$   & 45min 12s & 1$\times$   & 45min 27s & 1$\times$ \\
\hline
FedProx         & 7min 20s & 1.1$\times$ & 47min 39s & 1.1$\times$ & 49min 12s & 1.1$\times$ \\
\hline
FedGen          & 8min 11s & 1.2$\times$ & 50min 23s & 1.1$\times$ & 50min 31s & 1.1$\times$ \\
\hline
FedGKD          & 7min 5s  & 1.1$\times$ & 47min 43s & 1.1$\times$ & 48min 11s & 1.1$\times$ \\
\hline
q-FFL           & 7min 4s  & 1.1$\times$ & 47min 53s & 1.1$\times$ & 48min 3s  & 1.1$\times$ \\
\hline
\textbf{FedKF-} & 12min 3s & 1.8$\times$ & 84min 36s & 1.9$\times$ & 84min 40s & 1.9$\times$ \\
\hline
\textbf{FedKF}  & 12min 4s & 1.8$\times$ & 84min 38s & 1.9$\times$ & 84min 42s & 1.9$\times$ \\
\hline
\end{tabular}
}
\end{table}

\begin{table}[h]
\Huge
\centering
\caption{Time consumption needed in different solutions to achieve the same AMP as running FedAvg for 100 rounds on different datasets with $\alpha = 0.1$. The speedup of a solution is computed against FedAvg.}
\label{tab:cmp:time}
\resizebox{1\columnwidth}{!}
{
\renewcommand{\arraystretch}{1.4}
\begin{tabular}{c|cc|cc|cc}
\hline
\multirow{2}{*}{Solutions} & \multicolumn{2}{c|}{EMNIST}          & \multicolumn{2}{c|}{CIFAR-10}             & \multicolumn{2}{c}{CIFAR-100} \\
\cline{2-7}
                      & Time              & Speedup               & Time               & Speedup              & Time               & Speedup \\
\hline
\hline
FedAvg                & 6min 43s          & 1$\times$             & 45min 12s          & 1$\times$            & 45min 27s          & 1$\times$ \\
\hline
FedProx               & $>$6min 43s       & $<$1$\times$          & $>$45min 12s       & $<$1$\times$         & $>$45min 27s       & $<$1$\times$ \\
\hline
FedGen                & 6min 18s          & 1.1$\times$           & $>$45min 12s       & $<$1$\times$         & \textbf{37min 23s} & \textbf{1.2}$\times$ \\
\hline
FedGKD                & $>$6min 43s       & $<$1$\times$          & \textbf{23min 23s} & \textbf{1.9}$\times$ & 39min 31s          & 1.2$\times$ \\
\hline
q-FFL                 & $>$6min 43s       & $<$1$\times$          & $>$45min 12s       & $<$1$\times$         & $>$45min 27s       & $<$1$\times$ \\
\hline
\textbf{FedKF-} (ACA) & \textbf{3min 44s} & \textbf{1.8}$\times$  & 29min 37s          & 1.5$\times$          & $>$45min 27s       & $<$1$\times$ \\
\hline
\textbf{FedKF-} (OCA) & \textbf{3min 44s} & \textbf{1.8}$\times$  & 25min 23s          & 1.8$\times$          & $>$45min 27s       & $<$1$\times$ \\
\hline
\textbf{FedKF} (ACA)  & \textbf{3min 44s} & \textbf{1.8}$\times$  & 28min 47s          & 1.6$\times$          & $>$45min 27s       & $<$1$\times$ \\
\hline
\textbf{FedKF} (OCA)  & 3min 52s          & 1.7$\times$           & 23min 42s          & 1.9$\times$          & 44min 3s           & 1$\times$ \\
\hline
\end{tabular}
}
\end{table}

Table \ref{tab:cmp:time} reports the time consumption in different solutions to achieve the same AMP as running FedAvg for 100 rounds on different datasets with $\alpha=0.1$.
In FL, the total time consumption is obtained by (the training time per round)$\times$(number of rounds needed to converge). As shown in Table \ref{tab:cmp:time}, to achieve the same level of model performance, our solutions are generally faster than others on the EMINST and CIFAR-10 datasets and slower on the CIFAR-100 dataset. Thus, the computation overhead of FedKF and its variants are comparable to other solutions.

\noindent
\textbf{Robustness.}
It can be found from Fig. \ref{Fig::NumberofC} that FedKF is more robust than other solutions (in terms of performance stability) during FL training.
The fluctuation amplitude of AMP in FedKF (OCA) is extremely small compared with FedAvg, FedGen, FedGKD, and q-FFL when data is heterogeneous.

\begin{table}[h]
\centering
\caption{Performance comparison between FedAvg and FedAvg (OCA) on different datasets with $\alpha = 0.1$.}
\label{tab:improve}
{
\renewcommand{\arraystretch}{1.4}
\begin{tabular}{l|l|l|l}
\hline
Datasets & Metrics & FedAvg & FedAvg (OCA)\\
\hline
\hline
\multirow{3}{*}{EMNIST} & AMP (\%)              &  75.61 $\!\pm\!$ 0.92  & 78.38 $\!\pm\!$ 0.64 $\uparrow$ \\
\cline{2-4}
                        & FM ($\times 10^{-3}$) &  7.972 $\!\pm\!$ 0.703 & 5.329 $\!\pm\!$ 1.378 $\downarrow$ \\
\cline{2-4}
                        & WLP (\%)              &  57.07 $\!\pm\!$ 1.03  & 60.88 $\!\pm\!$ 1.87 $\uparrow$ \\
\hline
\multirow{3}{*}{CIFAR-10} & AMP (\%)            &  61.89 $\!\pm\!$ 0.81  & 65.82 $\!\pm\!$ 0.22 $\uparrow$ \\
\cline{2-4}
                        & FM ($\times 10^{-2}$) &  1.807 $\!\pm\!$ 0.240 & 1.194 $\!\pm\!$ 0.142 $\downarrow$ \\
\cline{2-4}
                        & WLP (\%)              &  34.41 $\!\pm\!$ 7.74  & 44.21 $\!\pm\!$ 4.58 $\uparrow$ \\
\hline
\multirow{3}{*}{CIFAR-100} & AMP (\%)           &  29.37 $\!\pm\!$ 0.58  & 34.76 $\!\pm\!$ 0.22 $\uparrow$ \\
\cline{2-4}
                        & FM ($\times 10^{-3}$) & 6.265 $\!\pm\!$ 1.206  & 1.853 $\!\pm\!$ 0.304 $\downarrow$ \\
\cline{2-4}
                        & WLP (\%)              &  17.70 $\!\pm\!$ 0.64  & 25.78 $\!\pm\!$ 0.92 $\uparrow$ \\
\hline
\end{tabular}
}
\end{table}

\noindent
\textbf{Using T1 to Improve FedAvg.}
As mentioned before, we can use T1 to improve prior solutions.
Let FedAvg (OCA) represent the solution using T1.
It simply uses the OCA model as the final model to be used.
Note that, during FedAvg (OCA) training, the global model broadcasted to active clients is still the ACA model.
Table \ref{tab:improve} and Fig. \ref{fig:improve} show the performance comparison between FedAvg and FedAvg (OCA) on different datasets with $\alpha = 0.1$.
In Table \ref{tab:improve}, it can be found that all metrics of FedAvg (OCA) are better than FedAvg.
Fig. \ref{fig:improve} demonstrates that FedAvg (OCA) has a faster learning speed and smaller fluctuation amplitude than FedAvg.

\begin{figure}[h]
\centering
\includegraphics[width=0.85\linewidth]{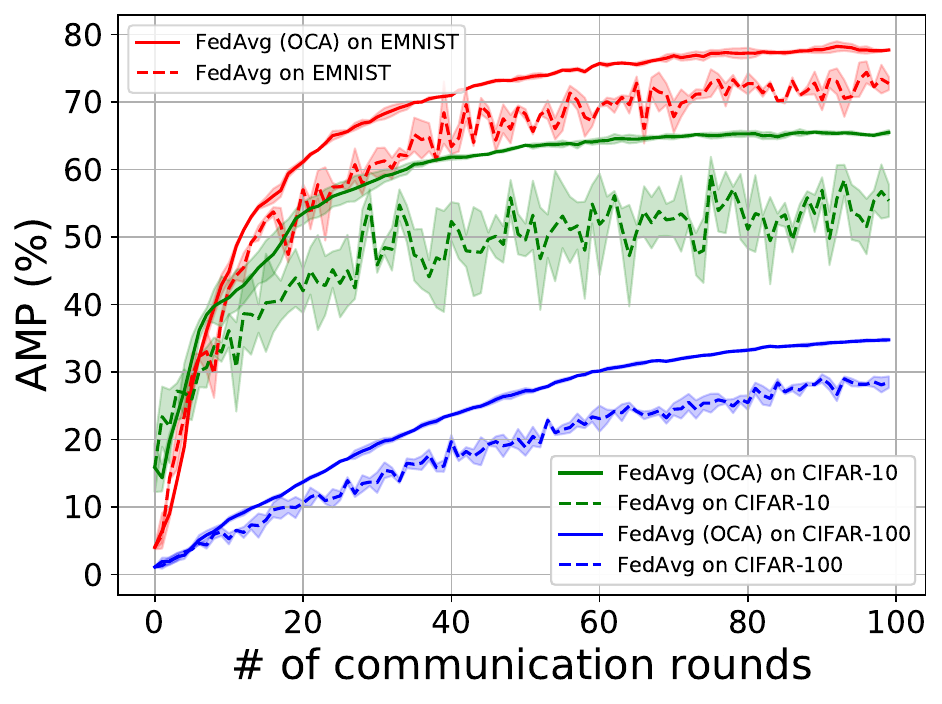}
\caption{AMP of FedAvg w/ and w/o T1 on different datasets with $\alpha = 0.1$.}
\label{fig:improve}
\end{figure}

To further explore under what circumstances T1 can improve FedAvg, we study the impacts of $\alpha$ (the concentration parameter of the Dirichlet distribution) and $K$ (the number of overall clients) on the performance of the global models (i.e., the ACA model and the OCA model).
Fig. \ref{fig:acaoca} shows the AMP of FedAvg v.s. $\alpha$ w/ and w/o T1 on different datasets with different $K$, where ``O" stands for the OCA model, ``A" stands for the ACA model, and $C$ represents the selection rate of active clients.
Note that $C=1$ represents full device participation that is unrealistic in most scenarios due to the presence of stragglers and the increase in communication cost per communication round (e.g., the communication cost per communication round of $C=1$ is 5 times as large as that of $C=0.2$).
In Fig. \ref{fig:acaoca}, it can be found that the AMP of $C=1$ is always the highest, which is consistent with the theoretical proof in \cite{li2019convergence}.
When the number of overall clients $K$ is limited, the more heterogeneous the data, the larger the gap in the AMP of the ACA model between $C=0.2$ and $C=1$.
In this case, compared to the ACA model, the OCA model (that additionally aggregates the latest historical local models of inactive clients) can obtain higher AMP and narrow the gap.

\begin{figure}[h]
\centering
\subfigure[EMNIST]{
    \centering
    \includegraphics[width=1\linewidth]{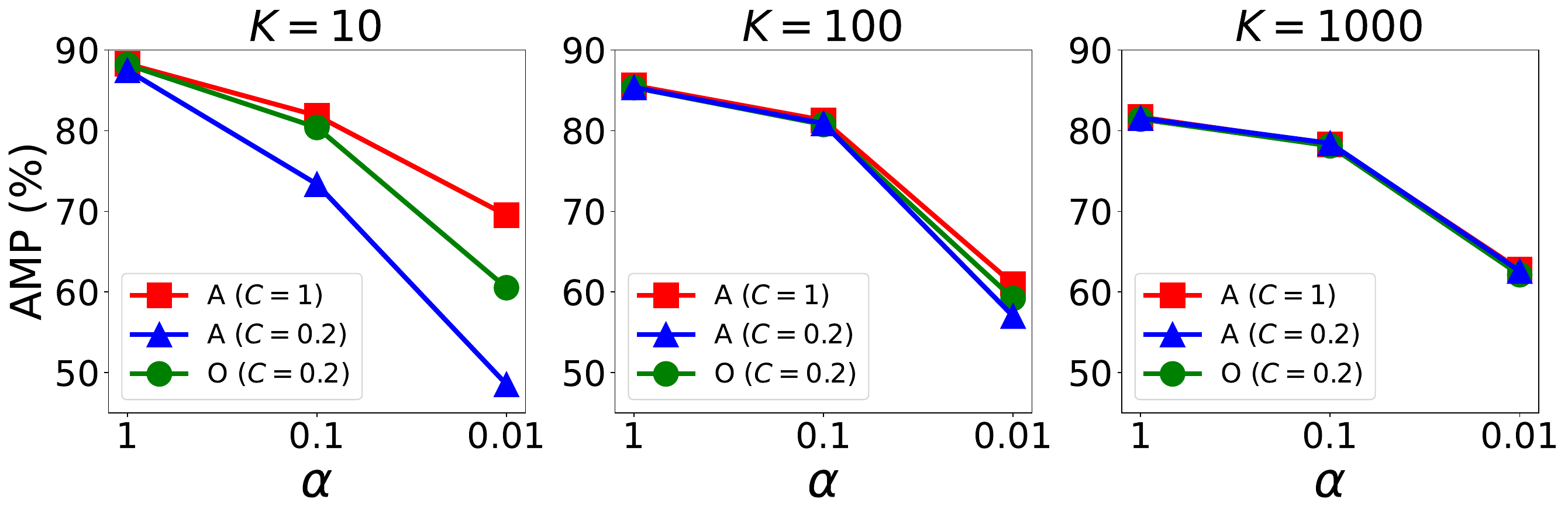}
}
\subfigure[CIFAR-10]{
    \centering
    \includegraphics[width=1\linewidth]{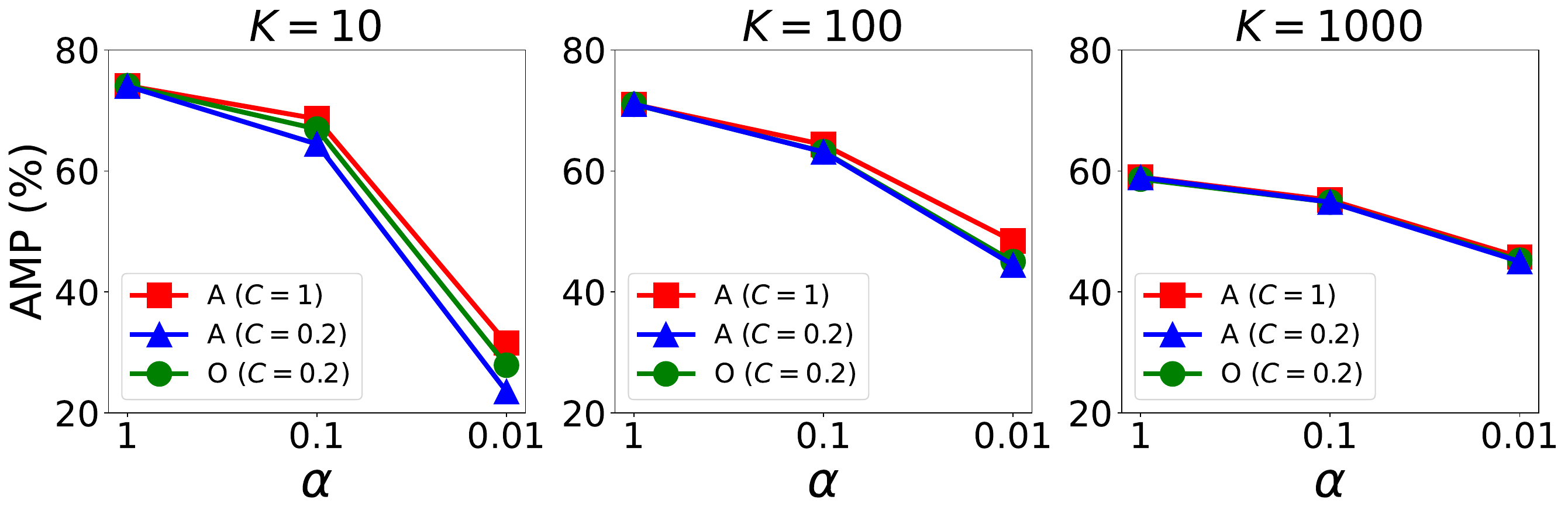}
}
\subfigure[CIFAR-100]{
    \centering
    \includegraphics[width=1\linewidth]{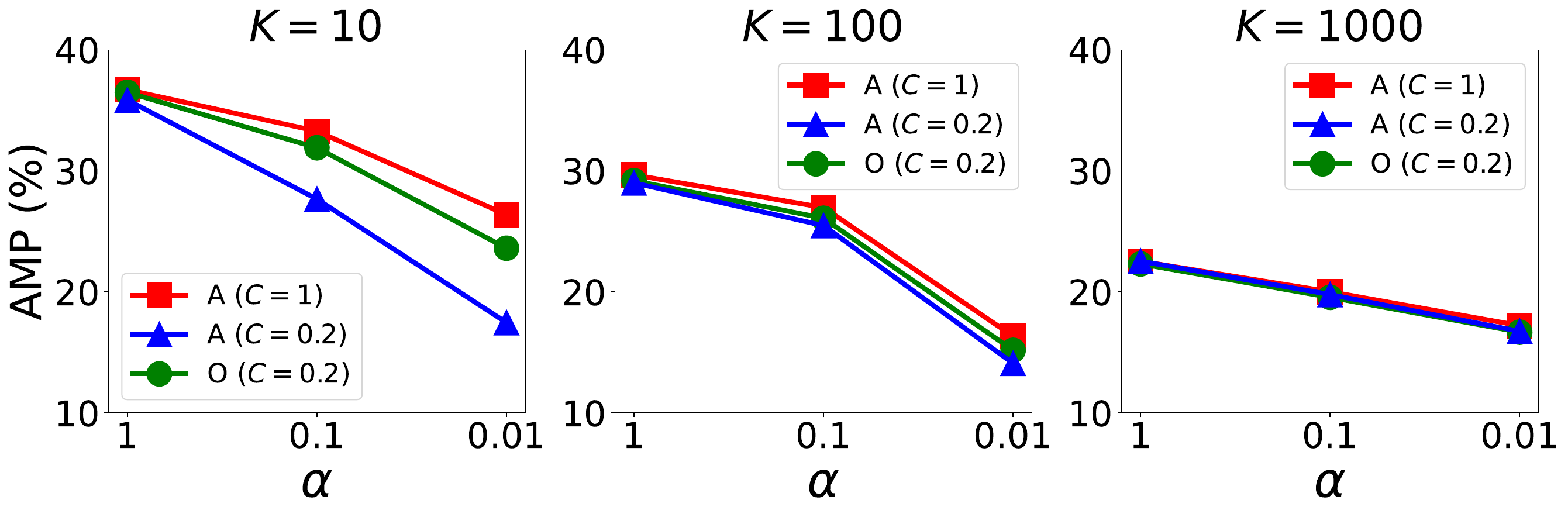}
}
\caption{AMP of FedAvg v.s. $\alpha$ w/ and w/o T1 on different datasets with different $K$.}
\label{fig:acaoca}
\end{figure}

\begin{figure}[h]
\centering
\includegraphics[width=1\linewidth]{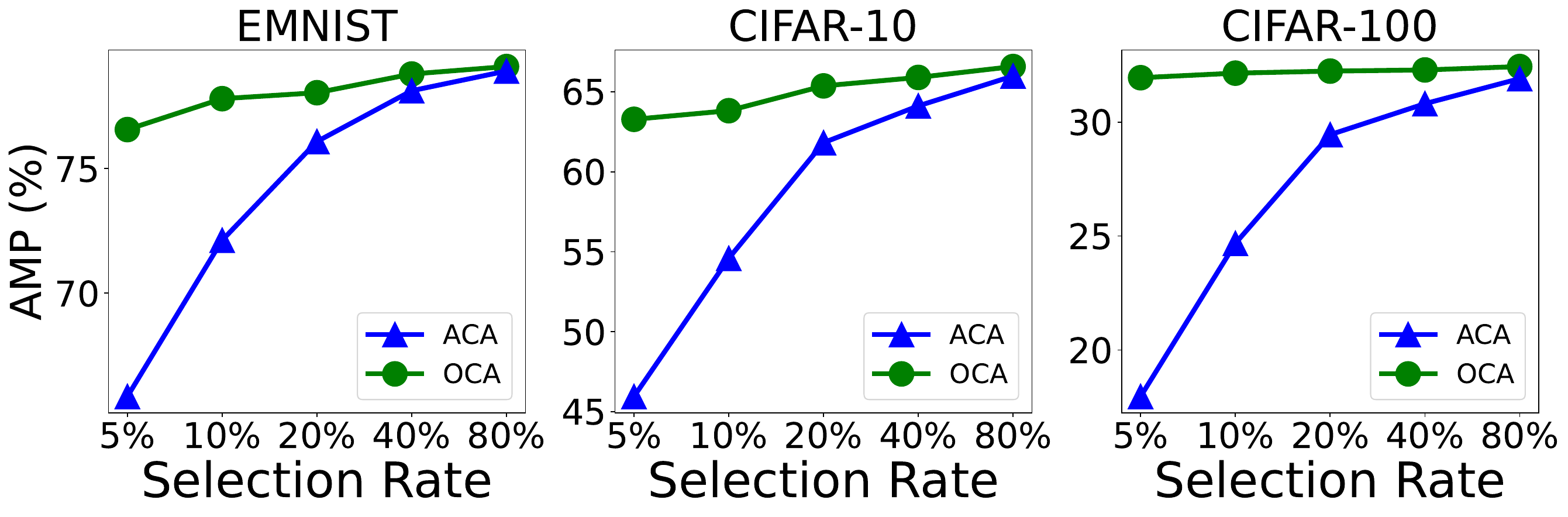}
\caption{AMP of FedAvg v.s. selection rate w/ and w/o T1 on different datasets with $\alpha = 0.1$.}
\label{fig:rate2ao}
\end{figure}

Fig. \ref{fig:rate2ao} shows the AMP of FedAvg v.s. selection rate w/ and w/o T1 on different datasets with $\alpha = 0.1$, where the ACA model is the aggregation of only active clients, and the OCA model is the aggregation of both active and inactive clients by using T1.
Note that when the selection rate is 5\%, only one client is selected as active in each round.
In Fig. \ref {fig:rate2ao}, it is easy to find that the performance (i.e., AMP) of the OCA model is always higher than that of the ACA model when the selection rate is less than 1.
Besides, the ACA model's performance degrades fast with the descent of the selection rate, i.e., by about 13\% on the EMNIST dataset, about 20\% on the CIFAR-10 dataset, and about 14\% on the CIFAR-100 dataset, which means the ACA model's performance is vulnerable to the low selection rate on non-IID data.
Fortunately, with T1, the OCA model's performance degrades slightly from 80\% to 5\% of the selection rate, especially on the CIFAR-100 dataset.
Therefore, the improvement using T1 (in terms of global model performance) is more obvious with a decreasing selection rate.

\section{Conclusion}
\label{conclusion}

In this paper, we have developed FedKF to handle data heterogeneity in FL.
Two novel techniques are developed to achieve the precise global knowledge representation and global-local knowledge fusion, by which the local model drift issue can be alleviated.
We theoretically prove that FedKF can directly turn out to be a good solution in heterogeneous agnostic FL, so FedKF has much broader application scenarios.
According to theoretical analysis and experimental results, FedKF achieves the three design goals (i.e., high model performance, high model fairness, and privacy-preserving) simultaneously.
In summary, the proposed techniques can mitigate data heterogeneity issues and significantly boost FL performance.
There are two directions to launch future work. 
First, we plan to test FedKF on more datasets and make the testbeds more diverse.
Second, we plan to improve FedKF further and make it more lightweight in terms of communication and computation overhead.

\bibliography{references}

\begin{thebibliography}{10}
\providecommand{\url}[1]{#1}
\csname url@samestyle\endcsname
\providecommand{\newblock}{\relax}
\providecommand{\bibinfo}[2]{#2}
\providecommand{\BIBentrySTDinterwordspacing}{\spaceskip=0pt\relax}
\providecommand{\BIBentryALTinterwordstretchfactor}{4}
\providecommand{\BIBentryALTinterwordspacing}{\spaceskip=\fontdimen2\font plus
\BIBentryALTinterwordstretchfactor\fontdimen3\font minus
  \fontdimen4\font\relax}
\providecommand{\BIBforeignlanguage}[2]{{%
\expandafter\ifx\csname l@#1\endcsname\relax
\typeout{** WARNING: IEEEtran.bst: No hyphenation pattern has been}%
\typeout{** loaded for the language `#1'. Using the pattern for}%
\typeout{** the default language instead.}%
\else
\language=\csname l@#1\endcsname
\fi
#2}}
\providecommand{\BIBdecl}{\relax}
\BIBdecl

\bibitem{zhou2023handling}
X.~Zhou, X.~Lei, C.~Yang, Y.~Shi, X.~Zhang, and J.~Shi, ``Handling data
  heterogeneity for iot devices in federated learning: A knowledge fusion
  approach,'' \emph{IEEE Internet of Things Journal (IoT-J)}, 2023.

\bibitem{mcmahan2017communication}
B.~McMahan, E.~Moore, D.~Ramage, S.~Hampson, and B.~A. y~Arcas,
  ``Communication-efficient learning of deep networks from decentralized
  data,'' in \emph{Artificial Intelligence and Statistics}, 2017, pp.
  1273--1282.

\bibitem{zheng2022exploring}
J.~Zheng, K.~Li, N.~Mhaisen, W.~Ni, E.~Tovar, and M.~Guizani, ``Exploring deep
  reinforcement learning-assisted federated learning for online resource
  allocation in privacy-preserving edgeiot,'' \emph{IEEE Internet of Things
  Journal (IoT-J)}, 2022.

\bibitem{wu2022fedadapt}
D.~Wu, R.~Ullah, P.~Harvey, P.~Kilpatrick, I.~Spence, and B.~Varghese,
  ``Fedadapt: Adaptive offloading for iot devices in federated learning,''
  \emph{IEEE Internet of Things Journal (IoT-J)}, 2022.

\bibitem{vu2022joint}
T.~T. Vu, D.~T. Ngo, H.~Q. Ngo, M.~N. Dao, N.~H. Tran, and R.~H. Middleton,
  ``Joint resource allocation to minimize execution time of federated learning
  in cell-free massive mimo,'' \emph{IEEE Internet of Things Journal (IoT-J)},
  2022.

\bibitem{al2022towards}
N.~Al-Maslamani, B.~S. Ciftler, M.~Abdallah, and M.~M. Mahmoud, ``Towards
  secure federated learning for iot using drl-enabled reputation mechanism,''
  \emph{IEEE Internet of Things Journal (IoT-J)}, 2022.

\bibitem{liu2022distributed}
Y.~Liu, Y.~Dong, H.~Wang, H.~Jiang, and Q.~Xu, ``Distributed fog computing and
  federated learning enabled secure aggregation for iot devices,'' \emph{IEEE
  Internet of Things Journal (IoT-J)}, 2022.

\bibitem{karimireddy2020scaffold}
S.~P. Karimireddy, S.~Kale, M.~Mohri, S.~Reddi, S.~Stich, and A.~T. Suresh,
  ``Scaffold: Stochastic controlled averaging for federated learning,'' in
  \emph{International Conference on Machine Learning (ICML)}, 2020, pp.
  5132--5143.

\bibitem{wang2020tackling}
J.~Wang, Q.~Liu, H.~Liang, G.~Joshi, and H.~V. Poor, ``Tackling the objective
  inconsistency problem in heterogeneous federated optimization,''
  \emph{Advances in Neural Information Processing Systems (NIPS)}, pp.
  7611--7623, 2020.

\bibitem{wang2021field}
J.~Wang, Z.~Charles, Z.~Xu, G.~Joshi, H.~B. McMahan, M.~Al-Shedivat, G.~Andrew,
  S.~Avestimehr, K.~Daly, D.~Data \emph{et~al.}, ``A field guide to federated
  optimization,'' \emph{arXiv preprint arXiv:2107.06917}, 2021.

\bibitem{li2020federated}
T.~Li, A.~K. Sahu, M.~Zaheer, M.~Sanjabi, A.~Talwalkar, and V.~Smith,
  ``Federated optimization in heterogeneous networks,'' \emph{Proceedings of
  Machine Learning and Systems (MLSys)}, pp. 429--450, 2020.

\bibitem{lin2020ensemble}
T.~Lin, L.~Kong, S.~U. Stich, and M.~Jaggi, ``Ensemble distillation for robust
  model fusion in federated learning,'' \emph{arXiv preprint arXiv:2006.07242},
  2020.

\bibitem{zhu2021data}
Z.~Zhu, J.~Hong, and J.~Zhou, ``Data-free knowledge distillation for
  heterogeneous federated learning,'' in \emph{International Conference on
  Machine Learning (ICML)}, 2021, pp. 12\,878--12\,889.

\bibitem{yao2021local}
D.~Yao, W.~Pan, Y.~Dai, Y.~Wan, X.~Ding, H.~Jin, Z.~Xu, and L.~Sun,
  ``Local-global knowledge distillation in heterogeneous federated learning
  with non-iid data,'' \emph{arXiv preprint arXiv:2107.00051}, 2021.

\bibitem{li2019fair}
T.~Li, M.~Sanjabi, A.~Beirami, and V.~Smith, ``Fair resource allocation in
  federated learning,'' in \emph{International Conference on Learning
  Representations (ICLR)}, 2019.

\bibitem{luo2021no}
M.~Luo, F.~Chen, D.~Hu, Y.~Zhang, J.~Liang, and J.~Feng, ``No fear of
  heterogeneity: Classifier calibration for federated learning with non-iid
  data,'' \emph{Advances in Neural Information Processing Systems (NIPS)}, pp.
  5972--5984, 2021.

\bibitem{li2021model}
Q.~Li, B.~He, and D.~Song, ``Model-contrastive federated learning,'' in
  \emph{Proceedings of the IEEE/CVF Conference on Computer Vision and Pattern
  Recognition (CVPR)}, 2021, pp. 10\,713--10\,722.

\bibitem{hu2022federated}
Z.~Hu, K.~Shaloudegi, G.~Zhang, and Y.~Yu, ``Federated learning meets
  multi-objective optimization,'' \emph{IEEE Transactions on Network Science
  and Engineering (TNSE)}, vol.~9, no.~4, pp. 2039--2051, 2022.

\bibitem{t2020personalized}
C.~T~Dinh, N.~Tran, and J.~Nguyen, ``Personalized federated learning with
  moreau envelopes,'' \emph{Advances in Neural Information Processing Systems
  (NIPS)}, pp. 21\,394--21\,405, 2020.

\bibitem{fallah2020personalized}
A.~Fallah, A.~Mokhtari, and A.~Ozdaglar, ``Personalized federated learning with
  theoretical guarantees: A model-agnostic meta-learning approach,''
  \emph{Advances in Neural Information Processing Systems (NIPS)}, pp.
  3557--3568, 2020.

\bibitem{ammad2019federated}
M.~Ammad-Ud-Din, E.~Ivannikova, S.~A. Khan, W.~Oyomno, Q.~Fu, K.~E. Tan, and
  A.~Flanagan, ``Federated collaborative filtering for privacy-preserving
  personalized recommendation system,'' \emph{arXiv preprint arXiv:1901.09888},
  2019.

\bibitem{tan2022towards}
A.~Z. Tan, H.~Yu, L.~Cui, and Q.~Yang, ``Towards personalized federated
  learning,'' \emph{IEEE Transactions on Neural Networks and Learning Systems
  (TNNLS)}, 2022.

\bibitem{konevcny2016federated}
J.~Kone{\v{c}}n{\`y}, H.~B. McMahan, F.~X. Yu, P.~Richt{\'a}rik, A.~T. Suresh,
  and D.~Bacon, ``Federated learning: Strategies for improving communication
  efficiency,'' \emph{arXiv preprint arXiv:1610.05492}, 2016.

\bibitem{hinton2015distilling}
G.~Hinton, O.~Vinyals, and J.~Dean, ``Distilling the knowledge in a neural
  network,'' \emph{arXiv preprint arXiv:1503.02531}, 2015.

\bibitem{fukuda2017efficient}
T.~Fukuda, M.~Suzuki, G.~Kurata, S.~Thomas, J.~Cui, and B.~Ramabhadran,
  ``Efficient knowledge distillation from an ensemble of teachers.'' in
  \emph{Interspeech}, 2017, pp. 3697--3701.

\bibitem{lopes2017data}
R.~G. Lopes, S.~Fenu, and T.~Starner, ``Data-free knowledge distillation for
  deep neural networks,'' \emph{arXiv preprint arXiv:1710.07535}, 2017.

\bibitem{chen2019data}
H.~Chen, Y.~Wang, C.~Xu, Z.~Yang, C.~Liu, B.~Shi, C.~Xu, C.~Xu, and Q.~Tian,
  ``Data-free learning of student networks,'' in \emph{Proceedings of the
  IEEE/CVF International Conference on Computer Vision (CVPR)}, 2019, pp.
  3514--3522.

\bibitem{fang2019data}
G.~Fang, J.~Song, C.~Shen, X.~Wang, D.~Chen, and M.~Song, ``Data-free
  adversarial distillation,'' \emph{arXiv preprint arXiv:1912.11006}, 2019.

\bibitem{li2019convergence}
X.~Li, K.~Huang, W.~Yang, S.~Wang, and Z.~Zhang, ``On the convergence of fedavg
  on non-iid data,'' in \emph{International Conference on Learning
  Representations (ICLR)}, 2019.

\bibitem{westerlund2019emergence}
M.~Westerlund, ``The emergence of deepfake technology: A review,''
  \emph{Technology Innovation Management Review}, vol.~9, no.~11, 2019.

\bibitem{lecun1989backpropagation}
Y.~LeCun, B.~Boser, J.~S. Denker, D.~Henderson, R.~E. Howard, W.~Hubbard, and
  L.~D. Jackel, ``Backpropagation applied to handwritten zip code
  recognition,'' \emph{Neural Computation}, pp. 541--551, 1989.

\bibitem{he2016deep}
K.~He, X.~Zhang, S.~Ren, and J.~Sun, ``Deep residual learning for image
  recognition,'' in \emph{Proceedings of the IEEE/CVF Conference on Computer
  Vision and Pattern Recognition (CVPR)}, 2016, pp. 770--778.

\bibitem{cohen2017emnist}
G.~Cohen, S.~Afshar, J.~Tapson, and A.~Van~Schaik, ``Emnist: Extending mnist to
  handwritten letters,'' in \emph{International Joint Conference on Neural
  Networks (IJCNN)}, 2017, pp. 2921--2926.

\bibitem{kullback1951information}
S.~Kullback and R.~A. Leibler, ``On information and sufficiency,'' \emph{The
  Annals of Mathematical Statistics}, pp. 79--86, 1951.

\bibitem{mohri2019agnostic}
M.~Mohri, G.~Sivek, and A.~T. Suresh, ``Agnostic federated learning,'' in
  \emph{International Conference on Machine Learning (ICML)}, 2019, pp.
  4615--4625.

\bibitem{krizhevsky2009learning}
A.~Krizhevsky, G.~Hinton \emph{et~al.}, ``Learning multiple layers of features
  from tiny images,'' \emph{Technical Report, University of Toronto}, 2009.

\bibitem{NEURIPS2019_9015}
A.~Paszke, S.~Gross, F.~Massa, A.~Lerer, J.~Bradbury, G.~Chanan, T.~Killeen,
  Z.~Lin, N.~Gimelshein, and e.~a. Antiga, Luca, ``Pytorch: An imperative
  style, high-performance deep learning library,'' in \emph{Advances in Neural
  Information Processing Systems (NIPS)}, 2019, pp. 8026--8037.

\bibitem{hsieh2020non}
K.~Hsieh, A.~Phanishayee, O.~Mutlu, and P.~Gibbons, ``The non-iid data quagmire
  of decentralized machine learning,'' in \emph{International Conference on
  Machine Learning (ICML)}.\hskip 1em plus 0.5em minus 0.4em\relax PMLR, 2020,
  pp. 4387--4398.

\bibitem{radford2015unsupervised}
A.~Radford, L.~Metz, and S.~Chintala, ``Unsupervised representation learning
  with deep convolutional generative adversarial networks,'' \emph{arXiv
  preprint arXiv:1511.06434}, 2015.

\end{thebibliography}

\bibliographystyle{IEEEtran}

\vfill

\end{document}